\documentclass{article} %
\usepackage{iclr2027_conference,times}

\usepackage{amsmath,amsfonts,bm,amsthm}

\newcommand{\taumin}{\tau_{\min}}
\newcommand{\vtaumin}{\vtau_{\min}}

\newcommand{\mat}[1]{\mathbf{#1}}

\newtheorem{proposition}{Proposition}
\newtheorem{corollary}{Corollary}
\theoremstyle{remark}

\def\eqref#1{equation~\ref{#1}}

\def\1{\bm{1}}

\def\vmu{{\bm{\mu}}}

\def\vb{{\bm{b}}}

\def\vs{{\bm{s}}}

\def\vx{{\bm{x}}}
\def\vy{{\bm{y}}}
\def\vz{{\bm{z}}}
\def\vtau{{\bm{\tau}}}

\def\veta{{\bm{\eta}}}
\def\vepsilon{{\bm{\varepsilon}}}

\def\vlambda{{\bm{\lambda}}}

\DeclareMathAlphabet{\mathsfit}{\encodingdefault}{\sfdefault}{m}{sl}
\SetMathAlphabet{\mathsfit}{bold}{\encodingdefault}{\sfdefault}{bx}{n}

\newcommand{\R}{\mathbb{R}}

\usepackage{hyperref}
\usepackage{url}
\newcommand{\diff}{\mathrm d}
\usepackage{float}

\usepackage{enumitem}
\usepackage{tabularx}
\usepackage{booktabs} %
\usepackage{graphicx}
\usepackage{bm}
\usepackage{hyperref}
\usepackage{cleveref}
\usepackage{multirow}
\usepackage{siunitx}
\usepackage{algorithm}
\usepackage{algorithmicx}
\usepackage{algpseudocode}
\usepackage{arydshln}
\usepackage{subcaption}

\title{DAWIS: Data Assimilation with Windowed\\Inverse Sampling 
via Multitask Interpolants
}
\iclrfinalcopy

\author{Erik Wikingsson$^{1}$\thanks{Equal contribution} \quad Martin Andrae$^{1}$\footnotemark[1] \quad Tomas Landelius$^{1,2}$ \quad Fredrik Lindsten$^{1}$ \\
$^1$Linköping University, Sweden \quad $^2$Swedish Meteorological and
Hydrological Institute\\
\texttt{erik.wikingsson@gmail.com, martin.andrae@liu.se} \\
\texttt{tomas.landelius@smhi.se, fredrik.lindsten@liu.se}
}

\begin{document}

\maketitle

\begin{abstract}
Flow- and diffusion-based generative models have recently emerged as flexible and highly efficient forecasting models for dynamical systems. When combined with inference-time guidance, they offer a promising route to high-dimensional non-Gaussian data assimilation (DA), the problem of combining forecasts with observations to estimate latent system states. Existing filters, however, condition on a fixed history and assimilate only the most recent observation, leaving them unable to revise past states when new observations arrive. Estimates then stay tethered to a history that later observations may contradict, and errors accumulate over the assimilation run. To this end, we introduce DAWIS,  a unified DA method covering filtering, fixed-lag smoothing, and block smoothing within a single framework. DAWIS replaces the single flow time of a state-level prior with a multitask stochastic interpolant over a window of consecutive states, assigning a separate flow time to each. An assimilation cycle inverts the window to a vector of per-state turning points and regenerates it under observation guidance, with the turning points controlling how strongly each state is held fixed, revised, or generated from scratch. The same construction can also absorb the forecast into the assimilation cycle, removing the need for a separate forecasting model. Experiments on challenging nonlinear systems show that DAWIS improves on both filtering and smoothing baselines under sparse, noisy, and nonlinear observations.
The code for DAWIS is available at \href{https://github.com/Erik-Wikingsson/DAWIS}{https://github.com/Erik-Wikingsson/DAWIS}
\end{abstract}

\section{Introduction}\label{intro}
Many scientific and engineering problems, from weather forecasting and fluid mechanics to neuroscience and robotics, require estimating latent states from incomplete and noisy observations \citep{asch_DA_2016}. 
Data assimilation (DA) combines prior knowledge of system dynamics with sparse, noisy observations $\vy_{1:n} := (\vy_1, \dots, \vy_n)$ of a latent state sequence $\vx_{0:n}$, where $\vy_n$ relates to $\vx_n$ through $p(\vy_n|\vx_n)$. 
DA comes in two complementary forms. {\em Filtering} tracks the {\em filtering distribution} $p(\vx_n|\vy_{1:n})$ online as observations arrive, whereas {\em smoothing} characterizes the {\em smoothing distribution} $p(\vx_{0:N}|\vy_{1:N})$ or its marginals $p(\vx_n|\vy_{1:N})$ over a fixed window, using future observations to refine past estimates.

Both are particularly challenging in high-dimensional systems such as the atmosphere, where states evolve chaotically and observations are sparse, noisy, and heterogeneous. 
Existing methods for approximating these distributions can be broadly grouped by their formulation of inference and uncertainty representation. 
{\em Variational methods}, such as 4DVar, estimate the most probable trajectory over an assimilation window via optimization. They scale to very high-dimensional systems but traditionally require tangent-linear and adjoint models and primarily provide point estimates. 
{\em Ensemble Kalman methods}, including the EnKF and ensemble Kalman smoothers, propagate an ensemble to approximate filtering or smoothing distributions through their means and covariances. They are computationally attractive but rely on Gaussian or approximately Gaussian representations and often require careful tuning of localization and inflation \citep{calvello2024accuracy,bannister2017review}. 
{\em Hybrid ensemble-variational methods}, such as En4DVar \citep{bonavita2012en4dvar}, combine variational optimization with flow-dependent ensemble covariances but retain limitations from finite ensembles and covariance-based uncertainty. 
Finally, {\em particle filters and smoothers} \citep{naesseth2019elements} can represent general distributions, but suffer from the curse of dimensionality \citep{bengtsson2008curse}. These limitations motivate flexible, data-driven methods capable of representing complex non-Gaussian distributions while remaining computationally tractable in high-dimensional systems.

Generative models provide an alternative route to high-dimensional sampling by representing flexible, non-Gaussian distributions and enabling posterior sampling through inference-time conditioning.
These ideas have been used to target filtering \citep{andrae2026daisi, chen2025flowdas} as well as smoothing \citep{rozet2023score, jia2026forcingdasunifiedrobustdata}. The choice of prior, however, largely dictates which of the two is available: methods that learn a stationary prior over states \citep{andrae2026daisi} or short trajectories \citep{rozet2023score} and those that learn a forecast model $p(\vx_{n+1}|\vx_n)$ \citep{chen2025flowdas, savary2025training} each commit to one regime. 
We generalize this construction from individual states to a {\em multitask trajectory prior}, assigning a separate flow time to each state in a window.
This brings filtering, fixed-lag smoothing, and block smoothing into a single framework and removes the need for an external forecast model.

\begin{figure}[t]
  \centering
  \includegraphics[width=\linewidth]{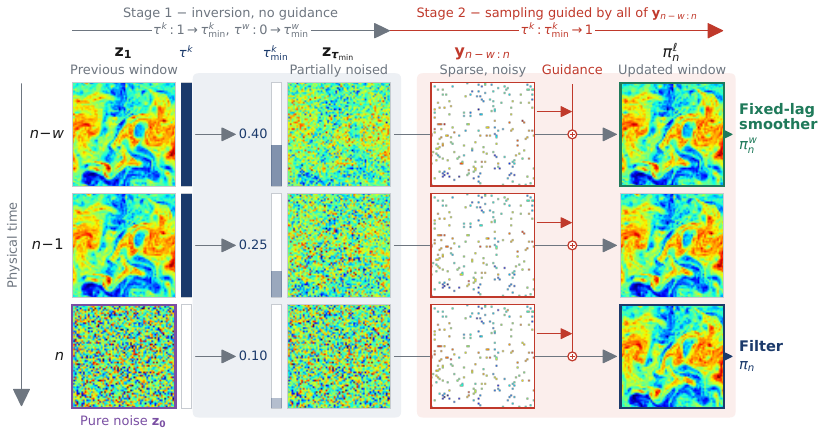}
  \vspace{-1.5em}
  \caption{
  The DAWIS assimilation cycle, here visualized for DAWIS-Joint
    with a prior of $w+1=3$ states. 
    (i)~\emph{Initialization:} The samples $\hat{\vx}_{n-w:n-1}$ from the previous cycle start at flow time $1$ and the new state $\vz_{\bm 0} \sim \rho_{\bm 0}$. 
    (ii)~\emph{Inversion:} the unconditional SDE moves each slot $k$ to its turning point $\taumin^k$, partially noising the old states and running the new one forward from noise. 
    (iii)~\emph{Guidance:} the guided SDE regenerates the window from $\vz_{\vtaumin}$, conditioned on observations $\vy_{n-w:n}$.}

  \label{fig:dawis}
\end{figure}

\paragraph{Contributions:}
We present DAWIS ({\bf D}ata {\bf A}ssimilation with {\bf W}indowed {\bf I}nverse {\bf S}ampling), a unified framework for data assimilation built on flow-based generative models:%
\begin{enumerate}[label=(\roman*)]
\item \textbf{Windowed inversion.} We extend the inversion-based assimilation of \cite{andrae2026daisi} from a single state to whole windows. Each cycle inverts the previous estimate to a vector of per-state turning points $\vtaumin$ and regenerates it under observation guidance, so new observations can revise past states instead of conditioning on a fixed history.
\item \textbf{A unified framework with one prior for several DA tasks.} Filtering, fixed-lag smoothing and block smoothing can all be done with the same trained model by changing the choice of $\vtaumin$. A single sliding-window cycle gives both a filter and a lagged smoother.
\item \textbf{Forecasting with or without an external model.} DAWIS can incorporate an existing forecasting model or, with DAWIS-Joint, generate forecasts within the same generative inference process, eliminating the need for a separate model.

\end{enumerate}

\section{Preliminaries}
\subsection{Data Assimilation}\label{sec:data_assimilation}
Consider the state-space model
\begin{align}\label{eq:ssm}
    \vx_0 \sim p(\vx_0), \qquad
    \vx_n \sim p(\vx_n \mid \vx_{n-1}), \qquad
    \vy_n \sim p(\vy_n \mid \vx_n), \qquad n = 1, \ldots, N,
\end{align}
where $\vx_n \in \mathbb{R}^d$ is the latent state and $\vy_n$ a noisy, partial observation of it. The states form a Markov chain and the observations are conditionally independent given the states. The transition is written as a density for convenience, but is only ever accessed through simulation and may equally be a deterministic simulator or a pre-trained generative model.

The aim of DA is to infer the latent states $\vx_{0:N}$ from the observations $\vy_{1:N}$. The different DA tasks we consider are distinguished by which observations are available when estimating each state. Given observations up to the present, the target is the {\em filtering distribution}
$\pi_n := p(\vx_n \mid \vy_{1:n})$,
computed sequentially as observations arrive. 
Allowing $\ell$ steps of look-ahead instead gives the {\em fixed-lag smoothing distribution}
$\pi^{\ell}_n := p(\vx_n \mid \vy_{1:n+\ell})$,
which recovers the filter at $\ell = 0$ and admits more future information as $\ell$ grows, with diminishing returns once $\ell$ exceeds the forgetting time of the dynamics~\citep{olsson2008sequential}. Using every observation gives the \emph{joint smoothing distribution}, which by \cref{eq:ssm} factorizes as
\begin{align}\label{eq:joint-posterior}
    p(\vx_{0:N} \mid \vy_{1:N}) \propto p(\vx_0) \prod_{n=1}^{N} p(\vx_n \mid \vx_{n-1}) \, p(\vy_n \mid \vx_n),
\end{align}
and whose marginals $p(\vx_n \mid \vy_{1:N})$ are the limit of the fixed-lag smoother as $\ell$ grows. Computing these targets exactly is intractable for the high-dimensional, nonlinear systems we consider. Our aim is to develop a practical approximation that cover all of them in a single unified framework.

\subsection{Multitask Stochastic Interpolants}
\label{sec:msi}

Our proposed DA method leverages recent advances in generative modeling. In the static setting (i.e., ignoring time evolution), flow-based models such as stochastic interpolants \citep{albergo2023stochastic} enable sampling from a target distribution $\rho_1$ by learning a transport map from a simple latent distribution $\rho_0$. In their standard form, these constructions transport the entire sample along a single time axis. 
However, when $\rho_1$ is a distribution over a trajectory of states $\vx_{n:n+w}$, we want access not only to the joint distribution but also to its conditionals. This can be achieved by assigning a separate time to each state in the window, allowing the states to be noised and denoised at different rates.
In this work, we adopt the framework of multitask interpolants \citep{negrel2025multitask}, though closely related formulations appear in the diffusion literature \citep{chen2024diffusion}. 

We define a {\em multitask interpolant} as the stochastic process
\begin{align}\label{eq:mt-interpolant}
  \vz_{\vtau} 
  = \alpha_{\vtau} \odot \vz_{\mathbf{0}} + \beta_{\vtau} \odot \vz_{\mathbf{1}}, \qquad \vtau \in [0,1]^{D},
\end{align}
where $\vz_\mathbf{0} \sim \rho_\mathbf{0}$ and $\vz_\mathbf{1} \sim \rho_\mathbf{1}$, and each entry of $\vtau$ specifies the {\em flow time} of the corresponding component. The scalar curves $\alpha, \beta \in C^2([0,1])$ satisfy $\alpha(0) = \beta(1) = 1$ and $\alpha(1) = \beta(0) = 0$ and all evaluations and multiplications are understood elementwise. 

In order to use the multitask interpolant for sampling, we must specify how the component-wise flow times evolve during generation. We therefore introduce a path $\vtau_t\in C^2([0,1];[0,1]^D)$ and consider the time-indexed interpolant $\vz_{\vtau_t}$. Different choices of $\vtau_t$ give rise to different sampling procedures. 
For example, $\vtau_t=t\bm{1}$ recovers the standard stochastic interpolant, 
while fixing $\tau_t^k\equiv1$ for $k\in\mathcal{K}$ and setting the corresponding data components to $\vz_1^k=\vx_k$ yields samples from $p(\vx\mid\vx_{\mathcal{K}})$.
Likewise, choosing a decreasing path ($\dot{\vtau}_t\leq 0$) reverses the transport, mapping data toward the latent space.

Let $\rho_{\vtau_t}:=\operatorname{Law}(\vz_{\vtau_t})$ denote the distribution of $\vz_{\vtau_t}$ with density $p_{\vtau_t}$, and define its {\em score} $\vs_\vtau(t, \vz) := \nabla \log p_{\vtau_t}(\vz)$ and {\em drift} $\vb_\vtau(t, \vz) := \mathbb{E}[\dot{\vz}_{\vtau_t} | \vz_{\vtau_t}=\vz]$. Then \citet{negrel2025multitask} show that for any non-negative $\vepsilon_\vtau\in C^0([0,1];\mathbb{R}_+^D)$, the SDE
\begin{subequations}\label{eq:forward-sde}
    \begin{align}
    &\diff \vz_{\vtau_t}= (\vb_{\vtau}(t, \vz_{\vtau_t}) + \vepsilon_{\vtau_t}\odot \vs_{\vtau}(t, \vz_{\vtau_t})) \diff t + \sqrt{2\vepsilon_{\vtau_t}}\odot\diff W_t, \quad \vz_{\vtau_0} \sim \rho_{\vtau_0}, \quad t \in [0, 1], \label{eq:forward-sde-ic}
    \end{align}
\end{subequations}
where $W_t$ is Brownian motion, has the same marginal laws $\rho_{\vtau_t}$ as the interpolant $\vz_{\vtau_t}$.
Sampling from \cref{eq:forward-sde} requires access to the drift and score. Rather than learning these quantities separately for each path \(\vtau_t\), we learn the path-independent conditional expectations
\begin{align}\label{eq:etas}
  \veta_{\bm{0}}(\vtau, \vz) := \mathbb{E}[\vz_{\mathbf{0}} \mid \vz_{\vtau} = \vz],
  \qquad
  \veta_{\bm{1}}(\vtau, \vz) := \mathbb{E}[\vz_{\mathbf{1}} \mid \vz_{\vtau} = \vz],
\end{align}
from which the drift and score can be recovered for any chosen path. The resulting sampling procedure therefore supports different paths without retraining. Details of the training objectives and the corresponding drift and score expressions are given in Appendix~\ref{app:interpolant}.

\subsection{Conditional Generation via Guidance}
\label{sec:cond-gen}
Given an observation $\vy \sim p(\vy | \vx)$ and the SDE in \cref{eq:forward-sde}, 
we can sample from the posterior $\rho^{\vy}_{\bm{1}}(\vx) \propto p(\vy | \vx) \rho_{\bm{1}}(\vx)$ by solving the following SDE with {\em guidance}:
\begin{subequations}\label{eq:forward-guided-sde}
    \begin{align}
    &\diff \vz_{\vtau_t}= (\vb_{\vtau}(t, \vz_{\vtau_t}; \vy) + \vepsilon_{\vtau_t}\odot \vs_{\vtau}(t, \vz_{\vtau_t}; \vy)) \diff t + \sqrt{2\vepsilon_{\vtau_t}}\odot\diff W_t, \quad \vz_{\vtau_0} \sim \rho_{\vtau_0}, \quad t \in [0, 1], \label{eq:forward-guided-sde-ic}
    \end{align}
\end{subequations}
where guided score and drift are given by
\begin{subequations}\label{eq:guided-drift}
    \begin{align}
    \vs_{\vtau}(t, \vz_{\vtau_t}; \vy) &= \vs_{\vtau}(t, \vz_{\vtau_t}) + \nabla_{\vz_{\vtau_t}} \log p(\vy | \vz_{\vtau_t}), \\
    \vb_{\vtau}(t, \vz_{\vtau_t}; \vy) &=
    \vb_{\vtau}(t, \vz_{\vtau_t}) + 
\vlambda_{\vtau_t} \odot \nabla_{\vz_{\vtau_t}} \log p(\vy | \vz_{\vtau_t}),
\end{align}
\end{subequations}
for $\vlambda_{\vtau_t}$ defined as in Appendix \ref{app:guidance}.
This also implies that the SDE \cref{eq:forward-guided-sde} share the same marginal laws $\rho_{\vtau_t}^\vy$ as the interpolant \cref{eq:mt-interpolant} where the endpoint $\vz_{\bm 1}$ is a sample from the posterior $\rho^{\vy}_{\bm{1}}$.

In general, the guidance term is intractable as we usually only have access to $p(\vy \mid \vz_{\bm{1}})$ but not $p(\vy \mid \vz_{\vtau_t}) = \mathbb{E}_{\vz_{\bm{1}} \mid \vz_{\vtau_t}}[p(\vy \mid \vz_{\bm{1}})]$. Various approaches exist to approximate this term, including diffusion posterior sampling (DPS) \cite{chung2023diffusion}, moment matching posterior sampling \citep[MMPS;][]{Rozet_Andry_Lanusse_Louppe_2024}, and recent asymptotically exact methods such as Meta Flow Maps \citep{potaptchik2026metaflowmapsenable}.
We use MMPS for its simplicity, but note that any gradient-based guidance method can be used within the DAWIS framework. See Appendix \ref{app:guidance} for details and \cite{daras2024survey} for a survey on guidance methods.

\section{Method}

DAWIS is a unified approach to filtering and smoothing in high-dimensional non-Gaussian DA, built on a single learned prior over windows of states $p(\vx_{n:n+w})$. 
The prior is parameterized as a multitask interpolant and is learned from trajectory data.
We assume a single flow time for each state within the window, giving a vector $\vtau = (\tau^0, \dots, \tau^w)$ of $w+1$ flow times. This enables generating a state from noise, holding it fixed, or anything in between, independently for each state in the window.
This flexibility allows us to simultaneously: (i) retain information from previous assimilation cycles for efficiency and temporal consistency, by \emph{partial} inversion and regeneration, and (ii) give the method sufficient flexibility to update old state variables as new data arrives.

DAWIS is an ensemble method that generates $J$ independent samples $\{\hat \vx_\cdot^j\}_{j=1}^J$ to approximate the filtering or smoothing distribution as an empirical distribution. We describe the methodology for generating a single sample $\hat \vx$. 
Each DAWIS assimilation cycle consists of three stages: initialization, inversion, and guidance. We build up the methodology by exemplifying these stages for filtering and fixed-lag smoothing, guided forecast models, and block smoothing, respectively, in the three subsections below. 
\Cref{alg:dawis} presents the general method and \cref{tab:modes} the design choices that recover the different instances.

\subsection{Filtering and Fixed-lag Smoothing}\label{sec:filter-smoother}
Consider first the filtering problem of sampling sequentially from $\pi_n$.
One option is to guide a generative forecast model trained to simulate from $p(\vx_n \mid \vx_{n-1})$ to target the conditional $p(\vx_n \mid \vy_n,  \vx_{n-1})$ (see further \cref{sec:guided-forecast}) \citep{chen2025flowdas, savary2025training}.
However, this requires conditioning the simulation on a current sample $\vx_{n-1}$, which makes the assimilation locked to a particular history. A remedy is to regenerate a longer window, simulating $p(\vx_{n-w+1:n} \mid \vy_{n-w+1:n}, \vx_{n-w})$ \citep{shysheya2024on}. However, this simulation is still based on a hard conditioning on a historical state, just further back in time, and furthermore requires regenerating the whole window from scratch.

An alternative is to use the \emph{inversion strategy} introduced in the DAISI method by \citet{andrae2026daisi}. They start by producing a forecast $\vx^{\rm f}_n$, using an off-the-shelf forecasting model, conditionally on the sample $\vx_{n-1}$. The forecast state is then inverted to a latent representation $\vz_{\taumin}$ by simulating the SDE in \eqref{eq:forward-sde} from $\tau_0 = 1$ to $\taumin$ (with the prior restricted to a single state). The inversion is followed by a guided "regeneration" of $\hat\vx_n$ using the guided SDE in \eqref{eq:forward-guided-sde}, running from $\taumin$ to $\tau=1$, while conditioning on $\vy_n$.

DAWIS generalizes this inversion strategy to windowed sampling, where the inversion of each state within the window can be controlled independently.
Let $\mathcal W = \{n-w, \dots, n\}$ denote the window at assimilation cycle $n$.
From the previous cycle, each ensemble member $\hat{\vx}_{n-w-1:n-1}$ is an approximate draw from $p(\vx_{n-w-1:n-1}\mid \vy_{1:n-1})$. We \textbf{initialize} cycle $n$ by shifting the sample by one time step and augmenting it with a forecast $\vx^{\rm f}_n$. As in DAISI, this can be obtained using an off-the-shelf forecast model, but an alternative is to use the DAWIS prior itself; see \cref{sec:guided-forecast} below.
This gives the initial state $\vz_{\bm 1} = [\hat{\vx}_{n-w:n-1}, \vx^{\rm f}_n]$, at  flow times $\vtau = \bm 1$. This follows from the fact that, with this initialization, all states within the window "live in the ambient state space", corresponding to $\tau = 1$.

Next, analogously to DAISI, the state $\vz_{\bm 1}$ is \textbf{inverted} 
to a latent representation.
However, for the \emph{windowed inversion}, each slot $k$ within the window is assigned its own {\em turning point} $\taumin^k \in [0,1]$, collected in $\vtaumin=(\taumin^0,\dots,\taumin^w)$. The unconditional SDE in \cref{eq:forward-sde} is integrated backwards from $\vtau = \bm 1$ to $\vtaumin$. The resulting latents $\vz_{\vtaumin}$ encode the information in the current sample in an intermediate latent space. 
Finally, we perform \textbf{guided} sampling given $\vy_{\mathcal W}$ by integrating the guided SDE in \cref{eq:forward-guided-sde} from $\vtaumin$ back to $\vtau = \bm 1$. This produces an assimilated sample $\hat{\vx}_{\mathcal W}$. 

This formulation combines the inversion strategy with windowed sampling. Importantly, by choosing a different turning point for each state, we have individual control over how far each state is inverted. Setting $\taumin^k = 1$ freezes slot $k$, while smaller values invert the slot further, allowing for more change.

The same procedure does not only produce filtering estimates, but also fixed-lag smoothing estimates. 
Because the window spans $w+1$ physical times, 
slot $w$ of the window holds the current state and gives an approximate filter $\vx_n \sim \pi_n$. Slot $0$, on the other hand, holds the state $w$ steps back and gives a fixed-lag smoother $\vx_{n-w} \sim \pi^w_{n-w}$, which has been constrained by the $w$ observations that arrived since. We refer to these as {\em DAWIS Filter} and {\em DAWIS Lagged Smoother}. The remaining slots slide into cycle $n+1$, so each state remains in the window for $w+1$ cycles and is revised by every observation that enters behind it.

\begin{algorithm}[t]
\caption{DAWIS -- one assimilation cycle}
\label{alg:dawis}
\begin{algorithmic}[1]
\Require current states $\vx_{\mathcal W}$, observations $\vy_{\mathcal W}$, start times $\vtau_0$, turning points $\vtaumin$
\State \textbf{Initialization:} Set $\vz_{\vtau_0} \gets \alpha_{\vtau_0} \odot \vz_{\bm 0} + \beta_{\vtau_0} \odot \vx_{\mathcal W}$ for $\vz_{\bm 0} \sim \rho_{\bm 0}$ %
\State \textbf{Inversion:} $\vz_{\vtaumin} \gets \textsc{Integrate}\big(\vz_{\vtau_0};\ \vtau_0 \to \vtaumin\big)$ \Comment{unconditional}
\State \textbf{Guidance:} $\vz_{\bm 1} \gets \textsc{Integrate}\big(\vz_{\vtaumin};\ \vtaumin \to \bm 1,\, \vy_{\mathcal W}\big)$
       \Comment{conditional on $\vy_{\mathcal W}$}
\State \Return $\hat{\vx}_{\mathcal W} \gets \vz_{\bm 1}$
\end{algorithmic}
\end{algorithm}

\begin{table}[tbh]
\centering
\small
\setlength{\tabcolsep}{4pt}
\caption{
Entries marked $\ast$ are irrelevant since $\beta_0 = 0$. $\vtaumin = \vtau_0$ means the inversion step is empty.
}
\label{tab:modes}
\begin{tabular}{lllll}
\toprule
Method & $\vx_{\mathcal W}$ & $\vtau_0$ & $\vtaumin$ & $\vy_{\mathcal W}$ \\
\midrule
DAWIS Filter / Lagged Sm. & $[\hat{\vx}_{n-w:n-1}, \vx^{\rm f}_n]$ & $(1, \dots, 1)$            & $\vtaumin$        & $\vy_{n-w:n}$ \\
DAWIS-Joint               & $[\hat{\vx}_{n-w:n-1}, \ast]$        & $(1, \dots, 1, 0)$         & $\vtaumin$        & $\vy_{n-w:n}$ \\
DAWIS Block Smoother      & $[\vx_{n:n+w}]$                        & $(1, \taumin, \dots, \taumin, 1)$ & $\vtau_0$ & $\vy_{\mathcal B}$ \\
Block Gibbs Smoother      & $[\vx_{n:n+w}]$                        & $(1, 0, \dots, 0, 1)$      & $\vtau_0$   & $\vy_{\mathcal B}$ \\
Joint AR $w|1$            & $[\hat{\vx}_{n-w}, \ast, \dots, \ast]$ & $(1, 0, \dots, 0)$       & $\vtau_0$   & $\vy_{n-w+1:n}$ \\
Joint AR $1|w$            & $[\hat{\vx}_{n-w:n-1}, \ast]$& $(1, \dots, 1,0)$       & $\vtau_0$   & $\vy_{n}$ \\
\bottomrule
\end{tabular}
\end{table}

\subsection{Guided Forecast models}
\label{sec:guided-forecast}

The procedure above assumes access to a forecast $\vx^{\rm f}_n$ from an external forecasting model $\mathcal{F}$, which in many applications may not be available. 
Since the multitask interpolant can sample from the conditional $p(\vx_{n} \mid \vx_{n-1})$, we could use this as a drop-in replacement for $\mathcal{F}$. However, this would require a separate sampling step, increasing the cost of a cycle.
Instead, inspired by guided forecast models, we propose to initialize the final slot $w$ from noise at flow time $\tau^w=0$, skipping the forecast stage entirely. During the inversion stage, this slot is now run \emph{forward} to $\taumin^w$, yielding a $\vz_{\vtaumin}$ similarly to before. The guidance step remains the same, generating the posterior directly without an extra step. We name this variant {\em DAWIS-Joint} and visualize it in \cref{fig:dawis}.

More generally, the SDE \cref{eq:forward-sde} can be initialized from any $\rho_{\vtau_0}$, allowing us to generate one or more states from pure noise. In particular, the windowed guided forecast model $p(\vx_{n-w+1:n} \mid \vy_{n-w+1:n}, \vx_{n-w})$ is itself an instance of DAWIS, obtained by freezing slot 0 at $\vx_{n-w}$, initializing $\vx_{n-w+1:n}$ from noise, and skipping the inversion stage. As for the DAWIS Lagged Smoother, slot 1 has been constrained by $\vy_{n-w+1:n}$, giving a fixed-lag smoother with lag $w-1$, $\vx_{n-w+1} \sim \pi^{w-1}_{n-w+1}$.

\subsection{Block Smoothing}\label{sec:block_smoothing}
The DAWIS Lagged Smoother performs sequential assimilation by sliding a window forward in time. Although it can utilize future observations, the fixed lag limits the possibility to propagate information from observations outside the window. %

Rather than retraining a prior over increasingly long windows, which can become infeasible for 
high-dimensional systems and requires a separate prior for each trajectory length, we introduce a block smoother that iteratively refines the trajectory while reusing the same prior.

For a window $\vx_{n:n+w}$ with interior $\mathcal B = \{n+1, \dots, n+w-1\}$ and boundary $\partial\mathcal B = \{n, n+w\}$, the Markov structure of \cref{eq:ssm} implies that
\begin{align}\label{eq:block-posterior}
  \pi_{\mathcal B}(\vx_{\mathcal B})
  := p(\vx_{\mathcal B} \mid \vx_{0:N \setminus \mathcal B}, \vy_{1:N})
  = p(\vx_{\mathcal B} \mid \vx_{\partial\mathcal B}, \vy_{\mathcal B})
  \propto p(\vx_{\mathcal B} \mid \vx_{\partial\mathcal B})\, p(\vy_{\mathcal B} \mid \vx_{\mathcal B}).
\end{align}
Repeatedly drawing from \cref{eq:block-posterior} for different blocks is a blocked Gibbs sampler for the joint smoothing distribution \cref{eq:joint-posterior} \citep{carter1994gibbs, singh2017blocking}. 
Each draw is available from the multitask interpolant by holding the boundary slots at $\vtau^{\partial\mathcal{B}}=\bm{1}$ and generating the interior from pure noise under guidance from $\vy_{\mathcal B}$. 
To ensure every state is updated, we use overlapping blocks, alternating forward and backward sweeps. Given exact guidance, this \emph{Block Gibbs Smoother} converges to \cref{eq:joint-posterior} from any initial trajectory \citep{carter1994gibbs}.

In practice, the guidance method will incur a bias, limiting the ability to converge to the true posterior despite many iterations. 
Instead, we use the block smoother to refine an already assimilated trajectory, using the alternating sweeps to propagate information from observations further away. However, since the interior is resampled from pure noise, much of the information is thrown away at every iteration, limiting its effectiveness. 

To remedy this, we propose the {\em DAWIS Block Smoother}, an extension of the DAWIS update to the block smoothing setting. Instead of discarding the previous sample, we initialize the interior $\vz^{\mathcal B}_{\vtaumin}$ from the interpolant 
\begin{align}\label{eq:block-noising}
  \vz^{\mathcal B}_{\vtaumin} = \alpha_{\taumin}\,\vz_{\bm 0} + \beta_{\taumin}\,\vx_{\mathcal B}, \qquad \vz_{\bm 0} \sim \rho_{\bm 0}.
\end{align}
preserving part of the information from the previous iteration. We then proceed the same way as before, running the guidance step conditioned on observations $\vy_{\mathcal{B}}$. 

For this to be a valid replacement of the Gibbs draw, it must leave the block posterior invariant under $\pi_{\mathcal{B}}$. 
To see that our update satisfies this, we let $\rho_\tau^\vy$ be the law of the interpolant \cref{eq:mt-interpolant} when the endpoints $\vz^{\partial\mathcal{B}}_{\bm 1} = \vx_{\partial\mathcal{B}}$ and the interior $\vz^\mathcal{B}_{\bm 1}\sim \pi_\mathcal{B}$. 
Under exact guidance, we know from \cref{sec:cond-gen} that the guided SDE \cref{eq:forward-guided-sde} shares the same marginals as this interpolant. Thus, it transports $\rho_\tau^\vy$ to $\rho_{\bm 1}^\vy=\pi_\mathcal{B}$, implying the invariance. For a formal proof of this, see \cref{prop:invariance}. We also note that similar observations have been made by \cite{hill2026exact, kang2026sampling}.

Together with the overlapping blocks above, the DAWIS Block Smoother is a valid MCMC scheme with \cref{eq:joint-posterior} as its stationary distribution under perfect guidance 
(see \cref{cor:sweep}). 
At $\taumin=0$, the noising discards the interior, recovering the Block Gibbs Smoother. Larger values of $\taumin$ act as a smaller step size, keeping the interior closer to its current value. This slows mixing, but shrinks the integration interval to $1-\taumin$, making each cycle cheaper and exposing it to less guidance error. In the experiments, we initialize from the DAWIS Lagged Smoother and run a few sweeps with $\taumin$ close to one, refining an already assimilated trajectory rather than sampling it from scratch.

\section{Related Works}
\paragraph{Guiding forecast models.}
Many ML-based filters guide a learned forecast model towards the incoming observations, sampling from $p(\vx_n \mid \vy_n, \vx_{n-1})$ \citep{chen2025flowdas, savary2025training} or from a longer window $p(\vx_{n-w+1:n} \mid \vy_{n-w+1:n}, \vx_{n-w})$ \citep{shysheya2024on}. Following \citet{shysheya2024on} we refer to these as {\em Joint AR $1|w$} and {\em Joint AR $w|1$}; both are instances of \cref{alg:dawis} with an empty inversion stage (\cref{tab:modes}). A softer variant treats $\vx_{n-w}$ as a trailing pseudo-observation rather than a hard anchor \citep{shysheya2024on}, used as the {\em SDA-Filter} baseline in \citet{andrae2026daisi}.

\paragraph{Decoupled forecast-analysis steps.}
Classical methods such as the Kalman filter and cycled 3DVar target filtering by alternating forecast and analysis steps \citep{kalnay2002da}. 
DAISI \citep{andrae2026daisi} adapts this to ML-based filters by coupling a forecast model with a stationary generative prior. DAWIS generalizes this inversion step to windows (\cref{sec:filter-smoother}).
Like the classical schemes, these assimilate only the latest observation, assuming past observations have already entered the estimate.
4DVar solves this by coupling several forecasting steps in a row, optimizing the initial state based on a window of observations \citep{kalnay2002da}. While extensions such as 
the Ensemble of Data Assimilations (EDA) with 4D-Var \citep{bonavita2012en4dvar}, referred to here as En4DVar, propagate an ensemble, the method still relies on a MAP estimate and requires inflation and perturbation to maintain ensemble spread.

\paragraph{All-at-once smoothing.}
Smoothing instead targets the full trajectory posterior $p(\vx_{0:N} | \vy_{1:N})$. Learning a flow-based prior over an entire trajectory is intractable, since the state dimension grows with $N$. Score-based Data Assimilation ({\em SDA}) \citep{rozet2023score} circumvents this by learning a prior only over local windows and coupling them at inference time, enabling parallel all-at-once sampling of trajectories of arbitrary length. The limited field of view comes at a cost: information propagates between distant states only through repeated local updates, so many solver steps and additional correction steps are needed before trajectories become coherent. With multimodal observations this can fail outright, as two ends of a trajectory may settle on incompatible modes. We compare SDA to DAWIS in the experiments.

\paragraph{Unified frameworks.}
Closest to our work are methods that assign an independent noise level to each state in a window, using a single model to cover several tasks. The idea originates in sequence generation, where rolling schedules noise later frames more heavily than earlier ones and the window slides forward during sampling \citep{pmlr-v235-ruhe24a, cachay2025elucidated, chen2024diffusion}. ForcingDAS \citep{jia2026forcingdasunifiedrobustdata} brings this construction to DA, learning a joint trajectory prior and selecting filtering, fixed-lag or full-sequence smoothing through the inference schedule. 
Its filter, however, is a conditional forecast that does not use the window to revise past states, and its full-sequence smoother requires a prior over the entire trajectory. 
Its fixed-lag smoother, {\em ForcingDAS-Pyr}, is closest to DAWIS, and we compare to a windowed version of it in our experiments. 
Like Joint AR $w|1$, it samples from $p(\vx_{n-w+1:n} \mid \vy_{n-w+1:n}, \vx_{n-w})$, conditioning on a fixed $\vx_{n-w}$.
But, instead of running each slot all the way to a clean state, it keeps the window at rolling noise levels and denoises each slot one level per cycle.
DAWIS shares the aim of one model for all tasks, but through inversion it revises all states in the window, including the nachor, and scales to full-trajectory smoothing without a costly full-sequence prior.

\section{Experiments}

We evaluate DAWIS on the SQG system and the \texttt{SEVIR} dataset with multiple observation operators for both filtering and smoothing. We present the CRPS for all experiments in \cref{tab:crps} and additional results in \cref{apx:additional_results}. We evaluate 10 trajectories with 20 ensemble members for each experiment and report both the mean and standard deviation. Our proposed method DAWIS consistently performs well across the whole range of experiments with different data and observation operators. Additional details about the experiments and the hyperparameters used is presented in \cref{apx:experiment_details}.
\begin{table}[tb]
\centering
\small
\caption{The CRPS for experiments on SQG and \texttt{SEVIR}.
We display the mean and standard deviation across 10 independent trajectories, averaged over the last 20 (10 for \texttt{SEVIR}) steps.
The best score for each experiment is highlighted in \textbf{bold} and the second best with an \underline{underline}.
Smoothing methods calculate the metrics over the whole trajectory.}
\label{tab:crps}
\setlength{\tabcolsep}{4pt}
\begin{tabular}{l c c c c c}
\toprule
\textbf{Method}
& \textbf{Noisy}
& \textbf{Sparse}
& \textbf{Multimodal}
& \textbf{Saturating}
& \textbf{SEVIR} \\

\midrule
\multicolumn{6}{l}{\textit{Filtering}} \\
\midrule

EnSF \citep{bao2024ensemble}
& $4.95{\scriptstyle \pm0.83}$
& $4.11{\scriptstyle \pm0.61}$
& $1.79{\scriptstyle \pm1.38}$
& $5.79{\scriptstyle \pm1.72}$
& $21.92{\scriptstyle \pm20.35}$ \\

FlowDAS \citep{chen2025flowdas}
& $2.73{\scriptstyle \pm0.50}$
& $3.22{\scriptstyle \pm0.56}$
& $2.50{\scriptstyle \pm0.48}$
& $3.66{\scriptstyle \pm0.41}$
& $15.66{\scriptstyle \pm7.17}$ \\

LETKF \citep{liang2025ensemble}
& $1.32{\scriptstyle \pm0.14}$
& $2.22{\scriptstyle \pm0.25}$
& $4.86{\scriptstyle \pm1.11}$
& $2.28{\scriptstyle \pm1.06}$
& $4.67{\scriptstyle \pm3.17}$ \\

DAISI \citep{andrae2026daisi}
& $1.30{\scriptstyle \pm0.14}$
& $1.69{\scriptstyle \pm0.20}$
& $1.55{\scriptstyle \pm0.09}$
& $1.79{\scriptstyle \pm0.45}$
& $3.97{\scriptstyle \pm2.65}$ \\

\addlinespace[0.2em]
\hdashline
\addlinespace[0.2em]

SDA-Filter \citep{andrae2026daisi}
& $1.32{\scriptstyle \pm0.13}$
& $1.75{\scriptstyle \pm0.25}$
& $2.03{\scriptstyle \pm0.32}$
& $4.26{\scriptstyle \pm0.70}$
& $\underline{3.79}{\scriptstyle \pm2.25}$ \\

Joint AR 1$|$W \citep{shysheya2024on}
& $3.83{\scriptstyle \pm0.41}$
& $4.01{\scriptstyle \pm0.54}$
& $3.97{\scriptstyle \pm0.41}$
& $3.77{\scriptstyle \pm0.44}$
& $6.42{\scriptstyle \pm3.86}$ \\

Joint AR W$|$1 \citep{shysheya2024on}
& $1.27{\scriptstyle \pm0.14}$
& $1.62{\scriptstyle \pm0.15}$
& $\underline{1.21}{\scriptstyle \pm0.10}$
& $1.20{\scriptstyle \pm0.16}$
& $4.28{\scriptstyle \pm2.73}$ \\

DAWIS Filter
& $\mathbf{1.03}{\scriptstyle \pm0.10}$
& $\mathbf{1.22}{\scriptstyle \pm0.14}$
& $\mathbf{1.04}{\scriptstyle \pm0.08}$
& $\mathbf{0.86}{\scriptstyle \pm0.13}$
& $3.86{\scriptstyle \pm2.12}$ \\

DAWIS-Joint Filter
& $\underline{1.18}{\scriptstyle \pm0.13}$
& $\underline{1.26}{\scriptstyle \pm0.16}$
& $\underline{1.21}{\scriptstyle \pm0.13}$
& $\underline{0.99}{\scriptstyle \pm0.16}$
& $\mathbf{3.58}{\scriptstyle \pm2.30}$ \\

\midrule
\multicolumn{6}{l}{\textit{Fixed-lag Smoothing}} \\
\midrule

SDA-Filter \citep{andrae2026daisi}
& $1.37{\scriptstyle \pm0.08}$
& $1.54{\scriptstyle \pm0.15}$
& $1.75{\scriptstyle \pm0.24}$
& $3.32{\scriptstyle \pm0.56}$
& $5.08{\scriptstyle \pm2.83}$ \\

ForcingDAS-Pyr \citep{jia2026forcingdasunifiedrobustdata}
& $1.36{\scriptstyle \pm0.08}$
& $1.28{\scriptstyle \pm0.11}$
& $\underline{1.05}{\scriptstyle \pm0.09}$
& $1.09{\scriptstyle \pm0.13}$
& $3.60{\scriptstyle \pm2.53}$ \\

Joint AR W$|$1 \citep{shysheya2024on}
& $1.13{\scriptstyle \pm0.06}$
& $1.29{\scriptstyle \pm0.06}$
& $1.06{\scriptstyle \pm0.05}$
& $1.05{\scriptstyle \pm0.10}$
& $3.59{\scriptstyle \pm2.61}$ \\

DAWIS Lagged Smoother
& $\mathbf{0.93}{\scriptstyle \pm0.05}$
& $\mathbf{1.05}{\scriptstyle \pm0.06}$
& $\mathbf{0.94}{\scriptstyle \pm0.04}$
& $\mathbf{0.76}{\scriptstyle \pm0.06}$
& $\underline{3.36}{\scriptstyle \pm2.03}$ \\

DAWIS-Joint Lagged Smoother
& $\underline{1.07}{\scriptstyle \pm0.06}$
& $\underline{1.08}{\scriptstyle \pm0.06}$
& $1.09{\scriptstyle \pm0.06}$
& $\underline{0.83}{\scriptstyle \pm0.07}$
& $\mathbf{3.34}{\scriptstyle \pm2.07}$ \\

\midrule
\multicolumn{6}{l}{\textit{Joint Smoothing}} \\
\midrule

SDA \citep{rozet2023score}
& $1.09{\scriptstyle \pm0.07}$
& $1.05{\scriptstyle \pm0.06}$
& $2.74{\scriptstyle \pm0.35}$
& $1.15{\scriptstyle \pm0.07}$
& $3.78{\scriptstyle \pm2.57}$ \\

Block Gibbs Smoother
& $\underline{0.92}{\scriptstyle \pm0.06}$
& $\underline{0.99}{\scriptstyle \pm0.07}$
& $\underline{0.90}{\scriptstyle \pm0.05}$
& $\underline{0.86}{\scriptstyle \pm0.09}$
& $\underline{3.58}{\scriptstyle \pm2.35}$ \\

DAWIS Block Smoother
& $\mathbf{0.90}{\scriptstyle \pm0.05}$
& $\mathbf{0.98}{\scriptstyle \pm0.06}$
& $\mathbf{0.89}{\scriptstyle \pm0.05}$
& $\mathbf{0.68}{\scriptstyle \pm0.06}$
& $\mathbf{3.18}{\scriptstyle \pm2.17}$ \\

\bottomrule
\end{tabular}
\end{table}

\paragraph{Surface Quasi-Geostrophic (SQG) Dynamics.}
To evaluate DAWIS, we consider a Surface Quasi-Geostrophic (SQG) system, which provides a compact yet challenging setting for studying data assimilation in turbulent flows. The model evolves a scalar field $\theta$ through nonlinear advection together with forcing and dissipative processes, resulting in complex multiscale dynamics and strong sensitivity to the initial state \citep{tulloch2009quasigeostrophic}. We use the SQG configuration from \citet{andrae2026daisi} to facilitate direct and fair comparison with previous results, discretizing the state on a $64\times64$ spatial grid and considering trajectories of 100 steps with a 3-hour interval between consecutive states. The observation scenarios, including the observation operators, noise levels, and spatial sparsity, are varied as specified in \cref{tab:experiments}. 
We use the numerical forward model for DAWIS Filter/Lagged Smoother and all other methods that require an external forward model.
For further details of the underlying model and simulation setup we refer to \citet{andrae2026daisi}.

\paragraph{Precipitation Nowcasting using \texttt{SEVIR}.}
We also evaluate DAWIS on the \texttt{SEVIR} dataset \citep{veillette2020sevir}, which consists of vertically integrated liquid radar observations of convective storms over the United States. Following the experimental setup of \citet{chen2025flowdas}, we use $384\times384$ spatial fields at \SI{3}{\kilo\meter} resolution with observations every \SI{10}{\minute}. We adopt their linear Gaussian observation setting with \SI{10}{\percent} spatial coverage and use the corresponding pretrained forecasting model for experiments that require an external forecast. Further experimental details are given by \citet{chen2025flowdas}. We use the FlowDAS \citep{chen2025flowdas} forward model for all methods that require an external forward model.

\paragraph{Filtering.}
DAWIS Filter has the lowest CRPS and RMSE of all filters in the four SQG settings, and DAWIS-Joint Filter is second in most of them (\cref{tab:crps,tab:rmse,fig:scorecard-crps,fig:scorecard-rmse,fig:crps-time-filtering}). On \texttt{SEVIR}, DAWIS-Joint Filter is the best filter and also improves on DAWIS Filter, which relies on the pretrained FlowDAS forecast model, showing that the forecast generated by the window prior is competitive with that of a dedicated forecast model. While the DAWIS models are robust across all experiments, several baselines break down in individual settings. LETKF struggles on \texttt{Multimodal}, where a Gaussian update cannot represent the bimodal posterior (\cref{fig:fields-filtering-multimodal}); SDA-Filter fails on \texttt{Saturating} (\cref{fig:fields-filtering-saturating}); EnSF and FlowDAS do not reproduce the high-frequency content of the fields (\cref{fig:spectra-filtering}).

\paragraph{Fixed-lag Smoothing.}
Both DAWIS/DAWIS-Joint Lagged Smoothers improve on the corresponding filters in every setting (\cref{tab:crps,tab:rmse,fig:crps-time-smoothing,fig:rmse-time-smoothing}). Since the two estimates come from the same assimilation cycle, the smoother requires no additional computation. DAWIS Lagged Smoother is the best fixed-lag smoother on all four SQG settings, and on \texttt{SEVIR} DAWIS-Joint Lagged Smoother has the lowest CRPS with DAWIS Lagged Smoother close behind.

\paragraph{Joint Smoothing.}
The block smoothers are initialized from the DAWIS Lagged Smoother for SQG and ther DAWIS-Joint Lagged Smoother for \texttt{SEVIR} and refine the trajectory with alternating forward and backward sweeps. The (approximative guidance-based) Block Gibbs Smoother, which regenerates each block from pure noise, gives small gains on \texttt{Noisy}, \texttt{Sparse}, and \texttt{Multimodal} but makes the initialization worse on \texttt{Saturating} and \texttt{SEVIR} (\cref{tab:crps,tab:rmse,fig:scorecard-crps,fig:scorecard-rmse}). We attribute this to the guidance error accumulating over the full regeneration path, whereas a turning point close to one keeps each update small. The DAWIS Block Smoother improves on the initialization in all settings, including \texttt{Saturating} and \texttt{SEVIR}, and has the lowest CRPS and RMSE of all methods on all experiments.

\begin{figure}[tb]
    \centering
    \begin{subfigure}[t]{\textwidth}
        \centering
        \includegraphics[width=\textwidth]{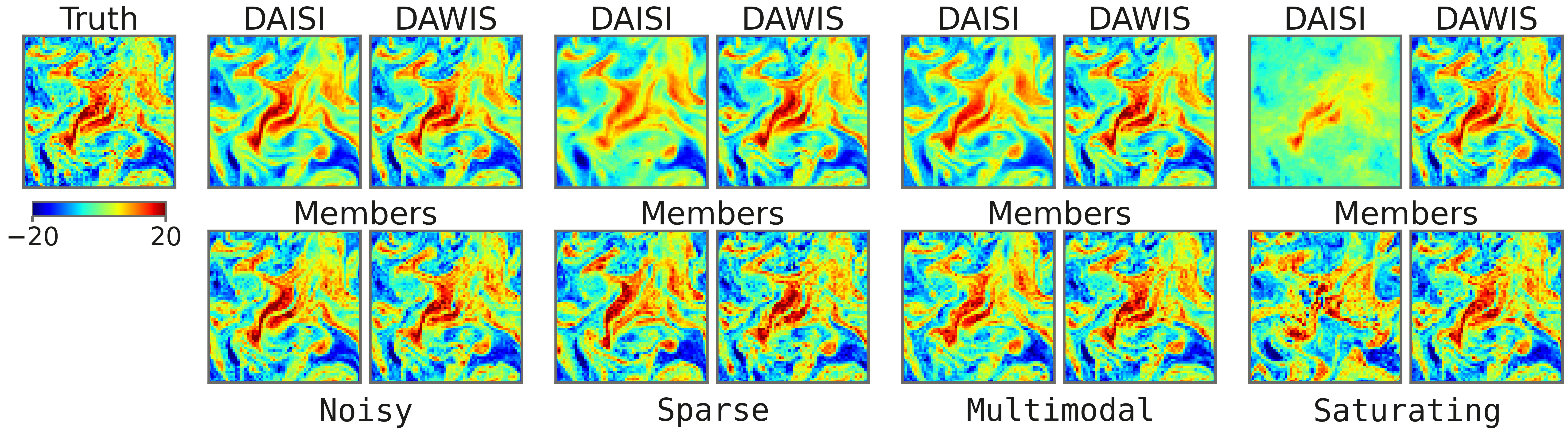}
        \caption{\textbf{Filtering:} DAISI vs DAWIS Filter }%
        \label{fig:states-filtering}
    \end{subfigure}
    \\[1em]
    \begin{subfigure}[t]{\textwidth}
        \centering
        \includegraphics[width=\textwidth]{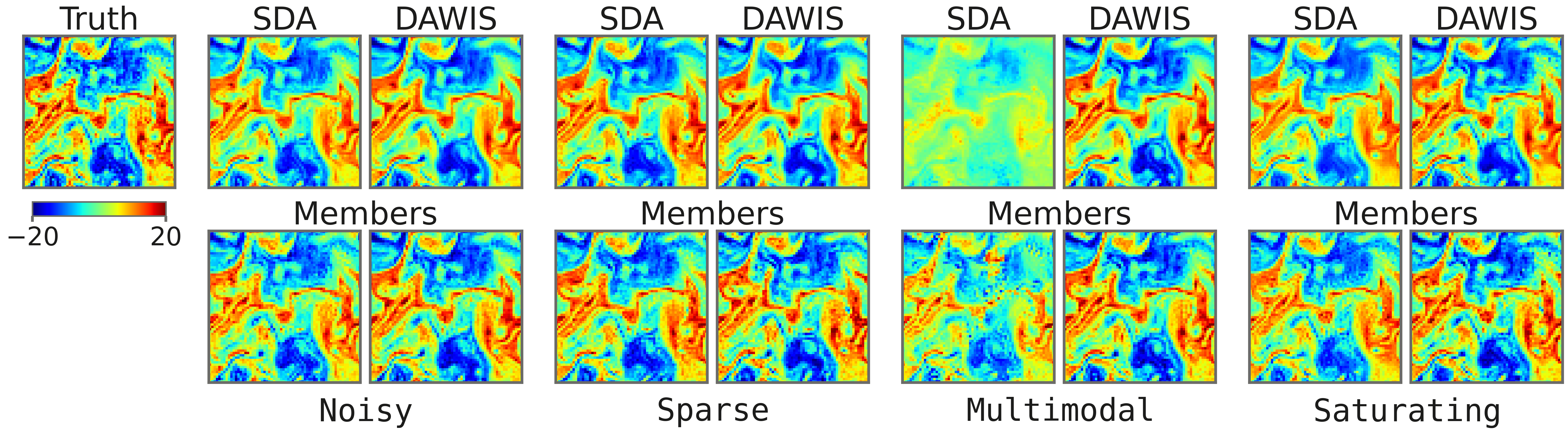}
        \caption{\textbf{Smoothing:} SDA vs DAWIS Block Smoother}%
        \label{fig:states-smoothing}
    \end{subfigure}
    \caption{
    Ensemble mean and a single member for all SQG experiments. Here visualized for the endpoint for filtering and the midpoint for smoothing.
    DAWIS exhibits robust performance for all settings, including the challenging multimodal and saturating observations where DAISI and SDA struggle in comparison.
    }
    \label{fig:states}
\end{figure}

\section{Conclusion}
We presented DAWIS, which treats data assimilation as inversion and guided regeneration of a window of states under a multitask trajectory prior. With a separate flow time for each state, the same trained model can perform filtering, fixed-lag smoothing, or block smoothing, with the task determined solely by the turning-point vector rather than retraining. The forecast can also be drawn from the prior itself, which removes the dependence on an external forecast model at a small cost in accuracy. Across multiple observation operators on SQG and on \texttt{SEVIR}, DAWIS demonstrates competitive skill and calibration for filtering, fixed lag smoothing, and block smoothing.

\paragraph{Limitations \& Future Work.}
Our evaluation covers a $64\times64$ SQG system and the \texttt{SEVIR} dataset with synthetic observations, a simplified setting compared to the state dimension of operational systems and without the heterogeneous, imperfectly known operators of a real observing network. Scaling DAWIS to higher-dimensional states and real observations is the natural next step.
Since the framework assumes nothing about the underlying dynamics, we also expect it to transfer to other domains where windowed generative priors are natural, such as computational fluid dynamics and video.
A second direction is sampling cost: each cycle inverts and regenerates a window of $W$ states using 100 solver steps, which could potentially be reduced using distillation approaches such as consistency models or flow maps. 
Finally, DAWIS inherits the limitations of whichever approximation of $p(\vy|\vz_{\tau_t})$ it is paired with. In our experiments, we use MMPS \citep{Rozet_Andry_Lanusse_Louppe_2024}, but more accurate guidance methods should transfer directly to DAWIS and improve the performance.

\subsection*{AI use statement}
In this work, we used generative AI tools to assist with translation. 
We have additionally used AI to formulate mathematical claims, provide critical ingredients for proving mathematical claims and assist in the writing of proofs.
We have not used generative AI tools to generate synthetic data sets, help develop theoretical models or conceptual frameworks, propose or refine hypotheses, design or provide feedback on research methodology or experiments, implement methods, clean and reformat dataset, support qualitative and thematic data analysis, interpret results. 
Additionally, we used generative AI tools to edit the research paper to improve readability and to assist software development. We have reviewed all AI-assisted work. LLM-generated code and text suggestions were verified and tested for correctness by the authors. We take responsibility for the final content of this work, including text, claims or artifacts produced with the aid of generative AI.

\subsection*{Ethics statement}
We have no ethical concerns or conflicts of interest to declare.

\subsection*{Reproducibility statement}
We have taken several steps to make our results reproducible. The DAWIS method is fully specified in \cref{alg:dawis}, with the design choices that recover each filter and smoother given in \cref{tab:modes} and described in \cref{sec:filter-smoother,sec:guided-forecast,sec:block_smoothing}. The notation is collected in \cref{app:notation}, and the multitask interpolant and the MMPS guidance, including the relation between drift, score and the learned conditional expectations, are derived in \cref{app:interpolant,app:guidance}. The theoretical results in \cref{app:theorems} state their assumptions (A1)--(A3) explicitly and come with complete proofs. \Cref{apx:experiment_details} gives the details needed to rerun our experiments: the DAWIS hyperparameters and how they were tuned (\cref{tab:experiment_hyperparams,tab:hyperparams_block}), the backbone architecture and optimizer settings (\cref{tab:optimizer}), the configuration and hyperparameters of every baseline (\cref{tab:LETKF_hyperparams,tab:hyperparams_flowdas,tab:hyperparams_ensf,tab:hyperparams_sda_filter,tab:hyperparams_daisi,tab:hyperparams_jointar,tab:hyperparams_sda}), and the data generation, train/validation/test splits, normalization and observation operators for SQG and \texttt{SEVIR} (\cref{tab:experiments,fig:obs}). All results are reported as the mean and standard deviation over 10 independent trajectories with 20 ensemble members (\cref{tab:crps,tab:rmse,tab:ssr}). The source code for training, assimilation and evaluation, together with the configuration files for all experiments, is available at \href{https://github.com/Erik-Wikingsson/DAWIS}{https://github.com/Erik-Wikingsson/DAWIS}.

\subsubsection*{Acknowledgments}
This research is financially supported by the Swedish Research Council (grant no: 2024-05011)
the Wallenberg AI, Autonomous Systems and Software Program (WASP) funded by the Knut and Alice Wallenberg Foundation,
and
the Excellence Center at Linköping--Lund in Information Technology (ELLIIT).
Our computations were enabled by the Berzelius resource at the National Supercomputer Centre, provided by the Knut and Alice Wallenberg Foundation.
Landelius was financially supported by the Swedish Foundation for Strategic Research.

\bibliography{iclr2027_conference}
\bibliographystyle{iclr2027_conference}

\newpage
\appendix
\section{Notation}\label{app:notation}
For reference, we collect the notation used throughout the paper. We keep physical (assimilation) time, indexed by $n$, separate from the flow time $\tau$ of the generative model. Subscripts index physical time and superscripts index positions within a window.

\paragraph{Indices and physical time.}
\begin{itemize}
  \item $n$: physical (assimilation) time index; $N$: final time index.
  \item $w$: window length, so a window holds $w+1$ states; $k = 0, \dots, w$: slot (position) within the window.
  \item $\vx_n \in \R^{d}$: physical state at time $n$; $d$ its dimension.
  \item   $\vx^{\rm f}_n$: forecast of $\vx_n$ from an external forecasting model.
  \item $\hat{\vx}_n \in \R^{d}$: current sample (analysis) of $\vx_n$ produced by an earlier cycle
  \item $\vx_{n:n+w} := (\vx_n, \dots, \vx_{n+w})$: state trajectory over a window; $D = (w+1)d$ its dimension.
  \item $\mathcal{W} = \{n, \dots, n+w\}$: the set of time indices in a window; $\vx_{\mathcal W}$, $\vy_{\mathcal W}$, $\hat{\vx}_{\mathcal W}$: the states, observations and updated estimate over it.
  \item $\vy_n$: observation at time $n$, with likelihood $p(\vy_n \mid \vx_n)$.
  \item $\mathcal{B}$: the interior of a block; $\partial\mathcal{B} \subseteq \{n, n+w\}$: its boundary, the window endpoints held fixed; $\vy_{\mathcal B}$: observations of the interior states.
  \item $J$: ensemble size; $j = 1, \dots, J$: ensemble member index.
\end{itemize}

\paragraph{Flow time and the multitask schedule.}
\begin{itemize}
  \item $\vtau \in [0,1]^{D}$: flow-time vector, one entry per component; $\tau = 0$ is noise, $\tau = 1$ is data. In DAWIS, $\vtau$ is constant within each state and we write $\vtau = (\tau^0, \dots, \tau^w)$ for the $w+1$ free flow times.
  \item $t \in [0,1]$: path parameter, always running from $0$ to $1$; not itself a flow time.
  \item $\vtau_t \in C^2([0,1]; [0,1]^D)$: the path along which the flow times evolve during sampling; $\dot{\vtau}_t$ its time derivative.
  \item $\vtaumin = (\taumin^0, \dots, \taumin^w)$: turning points, the flow times to which each slot is inverted before regeneration; $\taumin$: the common turning point of the interior slots in the block smoother.
  \item $\alpha_{\vtau}, \beta_{\vtau}$: interpolant coefficients, evaluated elementwise; $\alpha', \beta'$: their derivatives with respect to flow time.
  \item $\vz_{\bm 0} \sim \rho_{\bm 0}$, $\vz_{\bm 1} \sim \rho_{\bm 1}$: latent (noise) and data endpoints of the interpolant.
  \item $\vz_{\vtau}$: multitask interpolant \cref{eq:mt-interpolant}; $\vz_{\vtau_t}$ its value along a path. Superscripts select slots, e.g.\ $\vz^{k}_{\vtau}$, $\vz^{\mathcal B}_{\vtau}$, $\vz^{\partial\mathcal B}_{\bm 1}$.
  \item $\rho_{\vtau}$, $p_{\vtau}$: law and density of $\vz_{\vtau}$; $\rho^{\vy}_{\vtau}$: the same for the posterior interpolant, with $\rho^{\vy}_{\bm 1}$ the posterior target under guidance.
  \item $\vepsilon_{\vtau_t} \ge \bm 0$: diffusion coefficient of the generative SDE; $W_t$: Brownian motion.
  \item $\odot$: elementwise multiplication.
\end{itemize}

\paragraph{Distributions.}
\begin{itemize}
  \item $\pi_n := p(\vx_n \mid \vy_{1:n})$: filtering distribution.
  \item $\pi_{n}^{\ell} := p(\vx_n \mid \vy_{1:n+\ell})$: fixed-lag smoothing distribution with $\ell$ steps of look-ahead.
  \item $p(\vx_{0:N} \mid \vy_{1:N})$: joint smoothing distribution.
  \item $\pi_{\mathcal B} := p(\vx_{\mathcal{B}} \mid \vx_{\partial\mathcal{B}}, \vy_{\mathcal{B}})$: block posterior.
\end{itemize}

\section{Details on Multitask Stochastic Interpolants}
This appendix covers the multitask stochastic interpolant framework that DAWIS is built on. \Cref{app:interpolant} shows that the drift and score along any flow-time path follow from the same two learned conditional expectations. This is why a single network can be used for  filtering, fixed-lag smoothing and block smoothing without retraining. \Cref{app:guidance} extends these quantities to posterior sampling and gives the MMPS approximation of the likelihood score, which we use in all DAWIS experiments.

\subsection{Multitask Interpolants}
\label{app:interpolant}

Recall the multitask interpolant
\begin{align}\label{eq:interpolant-appendix}
  \vz_{\vtau} = \alpha_{\vtau} \odot \vz_{\bm 0} + \beta_{\vtau} \odot \vz_{\bm 1},
  \qquad \vz_{\bm 0} \sim \rho_{\bm 0}, \; \vz_{\bm 1} \sim \rho_{\bm 1}, \qquad \vtau \in [0,1]^D,
\end{align}
and the conditional expectations
\begin{align}\label{eq:etas-appendix}
  \veta_{\bm 0}(\vtau, \vz) := \mathbb{E}[\vz_{\bm 0} \mid \vz_{\vtau} = \vz],
  \qquad
  \veta_{\bm 1}(\vtau, \vz) := \mathbb{E}[\vz_{\bm 1} \mid \vz_{\vtau} = \vz].
\end{align}
Both are functions of the point $\vtau$ alone and do not depend on the path taken to reach it. 
Let $\vtau_t \in C^2([0,1])^D$ be a path and let $\rho_{\bm 0} = \mathcal{N}(\bm 0, \mat I)$. Then the drift and score of the interpolant \cref{eq:interpolant-appendix} along $\vtau_t$ are given by
\begin{align}
  \vb_{\vtau}(t, \vz) &= \dot{\vtau}_t \odot \Big( \alpha'_{\vtau_t} \odot \veta_{\bm 0}(\vtau_t, \vz) + \beta'_{\vtau_t} \odot \veta_{\bm 1}(\vtau_t, \vz) \Big), \label{eq:drift-from-etas} \\
  \vs_{\vtau}(t, \vz) &= -\alpha_{\vtau_t}^{-1} \odot \veta_{\bm 0}(\vtau_t, \vz), \label{eq:score-from-etas}
\end{align}
where $\alpha', \beta'$ denote derivatives of the scalar curves with respect to their own argument and all operations are elementwise. 

Two consequences are worth emphasizing. First, \cref{eq:drift-from-etas} and \cref{eq:score-from-etas} depend on the path only through the factor $\dot{\vtau}_t$ and through the point $\vtau_t$ at which $\veta_{\bm 0}, \veta_{\bm 1}$ are evaluated. A single model therefore serves every path, and no retraining is required when the schedule changes. Second, the initial condition is the law of the interpolant at $\vtau_0$ rather than pure noise, so a partially noised sample may be used as the starting point; this is what makes the inversion step of \cref{alg:dawis} possible.

In practice we learn $\veta_{\bm 0}, \veta_{\bm 1}$ by minimizing
\begin{align}
\mathcal{L}_i(\theta) =
    \mathbb{E}_{\vz_{\bm 0} \sim \mathcal{N}(\bm 0, \mat I), \, \vz_{\bm 1} \sim \rho_{\bm 1}, \, \vtau \sim \mathcal{U}}
    \Big[ \big\| \veta_{i}^\theta(\vtau, \vz_{\vtau}) - \vz_{i} \big\|^2 \Big],
    \qquad i \in \{\bm 0, \bm 1\},
\end{align}
whose minimizers are exactly the conditional expectations \cref{eq:etas-appendix}. In DAWIS the components of $\vtau$ are constant within each state, so the training distribution $\mathcal{U}$ draws $w+1$ independent times $\tau^k \sim \mathcal{U}[0,1]$, one per position in the window, and assigns $\tau^k$ to every component of state $k$. We predict $\veta_{\bm 0}$ and $\veta_{\bm 1}$ as separate output channels of a single network; see Appendix~\ref{app:training} for details.

\subsection{Guidance}\label{app:guidance}
Replacing the data distribution $\rho_{\mathbf 1}$ by the posterior $\rho_{\mathbf 1}^{\vy}$ leaves the conditional structure of the interpolant unchanged, so the conditional score and drift differ from their unconditional counterparts only through the likelihood score (see \cite{andrae2026daisi} for a derivation):
\begin{align}\label{eq:guided-score-drift}
  \vs^{\vy}(\vtau_t, \vz) &= \vs(\vtau_t, \vz_{\vtau_t}) + \nabla_{\vz_{\vtau_t}} \log p(\vy \mid \vz_{\vtau_t}),
\\
  \vb^{\vy}(\vtau_t, \vz_{\vtau_t}) &= \vb(\vtau_t, \vz_{\vtau_t}) + \vlambda_{\vtau_t} \odot \nabla_{\vz_{\vtau_t}} \log p(\vy \mid \vz_{\vtau_t}),
\end{align}
where the drift is rescaled componentwise by
\begin{align}\label{eq:lambda}
  \vlambda_{\vtau_t} :=  \frac{\alpha_{\vtau_t} \odot \gamma_{\vtau_t}}{\beta_{\vtau_t}},
  \qquad
  \gamma_{\vtau} := \dot{\vtau}_t \odot\left(\beta'_{\vtau} \odot \alpha_{\vtau} - \alpha'_{\vtau} \odot \beta_{\vtau}\right).
\end{align}
Substituting \cref{eq:guided-score-drift} into \cref{eq:forward-sde} gives the guided SDE \cref{eq:forward-guided-sde} of the main text, whose flow transports $\rho_{\vtau_0}$ to the posterior $\rho_{\mathbf 1}^{\vy}$. 

Exact conditional sampling still requires the likelihood score $\nabla_{\vz}\log p(\vy \mid \vz_{\vtau_t} = \vz)$, which is for the most part analytically intractable and requires approximation. Below we present MMPS, one such approximation. In addition we multiply the likelihood score by a guidance strength $\zeta > 0$, tuned per problem setup.

Throughout, let $\mat A_t := \text{diag}(\alpha_{\vtau_t})$ and $\mat B_t := \text{diag}(\beta_{\vtau_t})$ denote the interpolant coefficients of the current schedule point, and recall from Appendix~\ref{app:interpolant} that the learned network provides $\veta_1(\vtau_t, \vz) \approx \mathbb E[\vz_{\mathbf 1} \mid \vz_{\vtau_t} = \vz]$. Since the prior is defined over a window, $\vy$ collects the observations $\vy_{n-w:n}$ of that window and $\mathcal{H}$ acts blockwise, one observation operator per state.

\subsubsection{Moment-Matching Posterior Sampling (MMPS)}
MMPS approximates the denoising posterior by a Gaussian,
\begin{align}
    p(\vz_{\mathbf 1} \mid \vz_{\vtau_t} = \vz) \approx \mathcal N\big(\vmu_t(\vz), \, \mat\Sigma_t(\vz)\big),
    \qquad \vmu_t(\vz) := \veta_1(\vtau_t, \vz),
\end{align}
whose covariance $\mat\Sigma_t(\vz) := \mathbb E[\vz_{\mathbf 1}\vz_{\mathbf 1}^\top \mid \vz] - \vmu_t \vmu_t^\top$ is available from the Jacobian of the mean \citep{Rozet_Andry_Lanusse_Louppe_2024}:
\begin{align}\label{eq:mmps-cov}
    \mat\Sigma_t(\vz) = \nabla_{\vz} \vmu_t(\vz) \, \mat A_t^2 \mat B_t^{-1}.
\end{align}
For a single shared flow time, $\mat A_t^2\mat B_t^{-1}$ reduces to the scalar $\alpha_t^2/\beta_t$ and \cref{eq:mmps-cov} recovers the usual expression; with per-state flow times the factor is a diagonal matrix that does not commute with $\nabla_{\vz}\vmu_t$, so the right-multiplication matters.

Given a Gaussian likelihood $p(\vy \mid \vz_{\mathbf 1}) = \mathcal N(\vy \mid \mathcal H(\vz_{\mathbf 1}), \sigma_{\vy}^2 \mat I)$, marginalizing the Gaussian approximation yields
\begin{align}
    \nabla_{\vz} \log p(\vy \mid \vz)
    &= \nabla_{\vz} \log \left( \int_{\R^{D}} p(\vy \mid \vz_{\mathbf 1}) \, p(\vz_{\mathbf 1} \mid \vz) \diff \vz_{\mathbf 1} \right) \\
    &\approx \nabla_{\vz} \log \left( \int_{\R^{D}} p(\vy \mid \vz_{\mathbf 1}) \, \mathcal N\big(\vz_{\mathbf 1} \mid \vmu_t(\vz), \mat\Sigma_t(\vz)\big) \diff \vz_{\mathbf 1} \right) \\
    &\approx \big(\nabla_{\vz} \vmu_t\big)^{\!\top} \mat H_t^\top
    \Big( \sigma_{\vy}^2 \mat I + \mat H_t \, \nabla_{\vz}\vmu_t \, \mat A_t^2 \mat B_t^{-1} \mat H_t^\top \Big)^{-1}
    \big( \vy - \mathcal H(\vmu_t) \big),
    \label{eq:mmps-score}
\end{align}
where $\mat H_t := \nabla_{\vz_{\mathbf 1}} \mathcal H(\vz_{\mathbf 1}) \big\vert_{\vz_{\mathbf 1} = \vmu_t(\vz)}$. The final step neglects the dependence of $\mat\Sigma_t$ on $\vz$; the marginalization itself is exact whenever $\mathcal H$ is linear, though \cref{eq:mmps-score} remains an approximation of the true likelihood score because $p(\vz_{\mathbf 1} \mid \vz_{\vtau_t})$ is not Gaussian in general.

In practice neither $\nabla_{\vz}\vmu_t$ nor the inverse is formed explicitly: the solve is performed with conjugate gradients, requiring only Jacobian-vector products with $\veta_1$.

\section{Model Details}

Throughout we use the linear interpolant $\alpha_{\vtau} = \bm 1 - \vtau$, $\beta_{\vtau} = \vtau$, for which $\beta'\alpha - \alpha'\beta \equiv 1$. The drift and score \cref{eq:drift-from-etas}--\cref{eq:score-from-etas} then read
\begin{align}
  \vb_{\vtau}(t, \vz) = \dot{\vtau}_t \odot \big(\veta_{\bm 1}(\vtau_t, \vz) - \veta_{\bm 0}(\vtau_t, \vz)\big),
  \qquad
  \vs_{\vtau}(t, \vz) = -\frac{\veta_{\bm 0}(\vtau_t, \vz)}{\bm 1 - \vtau_t},
\end{align}
and \cref{eq:lambda} reduces to $\vlambda_{\vtau_t} = \dot{\vtau}_t \odot (\bm 1 - \vtau_t)/\vtau_t$.

We use the diffusion coefficient $\vepsilon_{\vtau_t} = \varepsilon\,\dot{\vtau}_t(\bm 1 - \vtau_t)$ for a scalar $\varepsilon \ge 0$. The $(\bm 1 - \vtau_t)$ is added to cancel the singularity of the score, and $\dot{\vtau}_t$ to ensure that fixed states do not recieve any added noise. We integrate both SDEs with Euler--Maruyama and use no additional Langevin correction steps.

\section{Theorems}
\label{app:theorems}

Throughout, we assume (A1) $\rho_{\bm 0} = \mathcal N(\bm 0, \mat I)$ 
and interpolant coefficients $\alpha, \beta$ as in \cref{sec:msi} with $\alpha,\beta > 0$ on $(0,1)$. %
, with all operations elementwise; (A2) the conditional expectations $\veta_{\bm 0}, \veta_{\bm 1}$ in \cref{eq:etas} and the likelihood score $\nabla_{\vz} \log p(\vy \mid \vz_{\vtau} = \vz)$ are exact; and (A3) the diffusion coefficient $\vepsilon_{\vtau_t} \ge \bm 0$ vanishes on every slot held at flow time $1$, so that frozen slots stay frozen (the score $-\alpha_{\vtau}^{-1} \odot \veta_{\bm 0}$ is singular at $\alpha = 0$). The linear interpolant $\alpha_{\vtau} = \bm 1 - \vtau$, $\beta_{\vtau} = \vtau$ used in the paper satisfies (A1) with $\beta'\alpha - \alpha'\beta \equiv 1$.

For the block smoother we allow the boundary $\partial\mathcal B \subseteq \{n, n+w\}$ of a window $\{n, \dots, n+w\}$ to contain only the endpoints that have a neighbour outside the window, so that the two edge blocks $n = 0$ and $n+w = N$ have a single boundary slot and $\vx_0$, $\vx_N$ can be updated; \cref{eq:block-posterior} holds verbatim for such blocks by the Markov property of \cref{eq:ssm}.

\begin{proposition}[Invariance of the block update]
\label{prop:invariance}
Fix a block with interior $\mathcal B$ and boundary $\partial\mathcal B$, and let $\pi_{\mathcal B} = p(\vx_{\mathcal B} \mid \vx_{\partial\mathcal B}, \vy_{\mathcal B})$ as in \cref{eq:block-posterior}. 
Let $K_{\taumin}$ be the DAWIS Block Smoother update, i.e.\ \cref{alg:dawis} with $\vtau_0 = \vtaumin$ equal to $1$ on $\partial\mathcal B$ and $\taumin$ on $\mathcal B$: the initialization \cref{eq:block-noising} with $\vz_{\bm 0} \sim \rho_{\bm 0}$ independent of $\vx_{\mathcal B}$, followed by the guidance stage, the guided SDE \cref{eq:forward-guided-sde} with likelihood $p(\vy_{\mathcal B} \mid \cdot)$ run along any path from $\vtaumin$ to $\bm 1$ that keeps $\partial\mathcal B$ at flow time $1$.
Then under \textnormal{(A1)}--\textnormal{(A3)}, for every $\taumin \in [0,1]$,
\begin{align}
  \pi_{\mathcal B} K_{\taumin} = \pi_{\mathcal B},
\end{align}
where $\pi K$ denotes the law of $\vx' \sim K(\vx, \cdot)$ for $\vx \sim \pi$. At $\taumin = 0$, $K_0(\vx_{\mathcal B}, \cdot) = \pi_{\mathcal B}$ for every $\vx_{\mathcal B}$, which is the Block Gibbs update. At $\taumin = 1$, $K_1$ is the identity.
\end{proposition}

\begin{proof}
Consider the interpolant \cref{eq:interpolant-appendix} on the window with the boundary slots held at flow time $1$, $\vz^{\partial\mathcal B}_{\bm 1} = \vx_{\partial\mathcal B}$, and interior endpoint $\vz^{\mathcal B}_{\bm 1} \sim \pi_{\mathcal B}$. Write $\rho^{\vy}_{\tau}$ for the law of its interior at common flow time $\tau$, so that $\rho^{\vy}_{1} = \pi_{\mathcal B}$ and $\rho^{\vy}_{0} = \rho_{\bm 0}$.

(a) Since $\alpha_1 = 0$ and $\beta_1 = 1$, the boundary slots satisfy $\vz^{\partial\mathcal B}_{\vtau} = \vz^{\partial\mathcal B}_{\bm 1}$, so conditioning on $\vz_{\vtau}$ with $\tau^{\partial\mathcal B} = 1$ conditions on $\vx_{\partial\mathcal B}$: on the interior, $\veta_{\bm 0}, \veta_{\bm 1}$ are the conditional expectations of the interpolant with data law $p(\vx_{\mathcal B} \mid \vx_{\partial\mathcal B})$. As $\pi_{\mathcal B} \propto p(\vx_{\mathcal B} \mid \vx_{\partial\mathcal B})\, p(\vy_{\mathcal B} \mid \vx_{\mathcal B})$, the interpolant above is the posterior interpolant of this conditional prior, and by \cref{app:guidance} its drift and score are the guided quantities \cref{eq:guided-score-drift}, which are exact under (A2). By \citet{negrel2025multitask} the SDE preserves its marginals for every admissible $\vepsilon$, and by (A3) the boundary does not move. Started from $\rho^{\vy}_{\taumin}$, the guidance stage thus returns a sample from $\rho^{\vy}_{1} = \pi_{\mathcal B}$.

(b) If $\vx_{\mathcal B} \sim \pi_{\mathcal B}$, the initialization returns $\alpha_{\taumin} \vz_{\bm 0} + \beta_{\taumin} \vx_{\mathcal B}$ with $\vz_{\bm 0} \sim \rho_{\bm 0}$ independent of $\vx_{\mathcal B}$, which is by definition a draw from $\rho^{\vy}_{\taumin}$.

Composing (b) and (a), $K_{\taumin}$ maps a sample from $\pi_{\mathcal B}$ to a sample from $\pi_{\mathcal B}$. At $\taumin = 0$ the initialization returns $\vz_{\bm 0}$ regardless of $\vx_{\mathcal B}$, so by (a) the output has law $\pi_{\mathcal B}$ for every input. At $\taumin = 1$, the initialization is the identity and the guidance stage is skipped.
\end{proof}

\begin{corollary}[Validity of the sweep]
\label{cor:sweep}
Apply the updates of \cref{prop:invariance} to a sequence of blocks, each acting on its interior given the current values of all other states. The composition leaves $p(\vx_{0:N} \mid \vy_{1:N})$ invariant. If every state is interior to some block (which requires the two edge blocks), $\taumin < 1$ and $\vepsilon > \bm 0$ on the interior, then under standard regularity conditions the chain is irreducible and aperiodic, so this is its unique stationary distribution \citep{tierney1994markov, meyn2012markov}.
\end{corollary}

\begin{proof}
If $\vx_{0:N} \sim p(\cdot \mid \vy_{1:N})$ then, conditionally on the states outside a block, its interior is distributed as $\pi_{\mathcal B}$ by \cref{eq:block-posterior}. \cref{prop:invariance} shows the update preserves this conditional law while leaving the other states unchanged, so the joint law is preserved, and a composition of invariant kernels is invariant. For $\taumin < 1$ the initialization is a Gaussian kernel with variance $\alpha_{\taumin}^2 > 0$ and for $\vepsilon > \bm 0$ the guidance stage is an SDE with non-degenerate noise, so both have strictly positive transition densities on the interior. Since every state is interior to some block, one sweep has a strictly positive density on the whole state space, which gives irreducibility and aperiodicity.
\end{proof}

\paragraph{Related results.}
\Cref{prop:invariance} is an instance of the observation, made in concurrent work by \citet{kang2026sampling} and \citet{hill2026exact}, that a Gaussian noising step followed by exact denoising leaves the data distribution invariant; Gibbs samplers alternating noising and guided denoising with diffusion priors also appear in \citet{coeurdoux2024plug} and \citet{wu2024principled}. The blocked structure of \cref{cor:sweep} follows the block Gibbs samplers for state-space models of \citet{carter1994gibbs}, \citet{shephard1997likelihood} and \citet{singh2017blocking}.

\paragraph{Relation to the inversion stage.}
The block smoother initializes the interior directly at $\vtaumin$ through \cref{eq:block-noising} rather than inverting it from $\bm 1$. The reason is that the unconditional inversion of \cref{alg:dawis} does not in general yield a $\pi_{\mathcal B}$-invariant kernel: it transports the interior along the marginals of the prior interpolant, whereas invariance requires reaching the posterior marginal $\rho^{\vy}_{\taumin}$. Inverting with the guided SDE instead does leave $\pi_{\mathcal B}$ invariant, by the same marginal-preservation argument as in step (a) of \cref{prop:invariance}. The next result shows that for a specific choice of diffusion coefficient the conditional and unconditional inversions coincide, and that both reduce to the initialization \cref{eq:block-noising}. The block smoother is therefore a special case of the inversion cycle rather than a departure from it. The underlying fact, that the generative SDE with this coefficient reduces to the noising step of SDEdit \citep{meng2022sdedit}, is not new, a similar result is Proposition~H.1 of \citet{potaptchik2026metaflowmapsenable}.

\begin{proposition}[Score-free inversion]
\label{prop:sdedit}
Let $\vtau_t$ be a path with $\vtau_0 = \bm 1$, $\dot{\vtau}_t \le \bm 0$ componentwise and $\vtau_t > \bm 0$,  and assume $\beta'\alpha - \alpha'\beta \ge 0$. Let $\vepsilon_{\vtau_t} = -\vlambda_{\vtau_t}$ with $\vlambda_{\vtau_t}$ as in \cref{eq:lambda}. Then $\vepsilon_{\vtau_t} \ge \bm 0$, the SDEs \cref{eq:forward-sde} and \cref{eq:forward-guided-sde} coincide along $\vtau_t$, and the resulting SDE involves neither $\veta_{\bm 0}, \veta_{\bm 1}$ nor the likelihood. Moreover, its solution started from $\vz$ at $t = 0$ has, for every $t$, the law of the interpolant \cref{eq:interpolant-appendix} with $\vz_{\bm 1} = \vz$,
\begin{align}\label{eq:app-noising}
  \vz_{\vtau_t} \overset{d}{=} \alpha_{\vtau_t} \odot \vz_{\bm 0} + \beta_{\vtau_t} \odot \vz, \qquad \vz_{\bm 0} \sim \rho_{\bm 0} \text{ independent of } \vz,
\end{align}
which extends by continuity to $\vtau_t = \bm 0$, where it equals $\vz_{\bm 0}$.
\end{proposition}

\begin{proof}
Since the drift, score and diffusion coefficient act elementwise and the components of $\vz_{\bm 0}$ are independent, it suffices to treat a single component. We write $\tau_t$ for its flow time, drop the argument $\tau_t$ from $\alpha, \beta, \eta_{\bm 0}, \eta_{\bm 1}, \lambda, \varepsilon$, and let $\ell(z) := \partial_z \log p(\vy \mid z_{\tau_t} = z)$ denote the likelihood score.

By \cref{eq:lambda} and (A1), $\varepsilon = -\lambda = -\dot\tau_t\, \alpha(\beta'\alpha - \alpha'\beta) / \beta \ge 0$, since $\dot\tau_t \le 0$, $\alpha \ge 0$, $\beta > 0$ and $\beta'\alpha - \alpha'\beta \ge 0$. Taking $\mathbb E[\,\cdot \mid z_{\tau_t} = z]$ in \cref{eq:interpolant-appendix} gives $z = \alpha\,\eta_{\bm 0} + \beta\,\eta_{\bm 1}$, and substituting $\eta_{\bm 1} = (z - \alpha \eta_{\bm 0})/\beta$ into \cref{eq:drift-from-etas}--\cref{eq:score-from-etas}, the drift of \cref{eq:forward-sde} becomes
\begin{align}\label{eq:app-drift}
  b + \varepsilon\, s
  = \dot\tau_t\frac{\beta'}{\beta}\,z
    - \frac{\eta_{\bm 0}}{\alpha}\Big(\varepsilon + \dot\tau_t\frac{\alpha(\beta'\alpha - \alpha'\beta)}{\beta}\Big)
  = \dot\tau_t\frac{\beta'}{\beta}\,z - \frac{\eta_{\bm 0}}{\alpha}\big(\varepsilon + \lambda\big),
\end{align}
where the last step is the definition \cref{eq:lambda}. By \cref{eq:guided-score-drift}, the guided SDE adds $(\lambda + \varepsilon)\,\ell$ to this drift. At $\varepsilon = -\lambda$ both the $\eta_{\bm 0}$ term and the likelihood term vanish, so the two SDEs reduce to the same linear SDE
\begin{align}\label{eq:app-linear}
  \diff z_{\tau_t} = \dot\tau_t\frac{\beta'}{\beta}\,z_{\tau_t}\,\diff t + \sqrt{-2\lambda}\,\diff W_t,
\end{align}
with continuous coefficients since $\beta > 0$.

Being linear with deterministic coefficients, the solution of \cref{eq:app-linear} from a fixed $z$ is Gaussian, so it suffices to match its mean $m_t$ and variance $v_t$ with \cref{eq:app-noising}. These satisfy
\begin{align}\label{eq:app-mv}
  \dot m_t = \dot\tau_t\frac{\beta'}{\beta}\,m_t,
  \qquad
  \dot v_t = 2\dot\tau_t\frac{\beta'}{\beta}\,v_t - 2\dot\tau_t\frac{\alpha(\beta'\alpha - \alpha'\beta)}{\beta},
  \qquad m_0 = z,\ v_0 = 0.
\end{align}
By the chain rule, $m_t = \beta_{\tau_t} z$ and $v_t = \alpha_{\tau_t}^2$ satisfy $\dot m_t = \dot\tau_t \beta' z = \dot\tau_t \frac{\beta'}{\beta} m_t$ and $\dot v_t = 2\dot\tau_t \alpha\alpha' = 2\dot\tau_t\frac{\beta'}{\beta}\alpha^2 - 2\dot\tau_t\frac{\alpha(\beta'\alpha - \alpha'\beta)}{\beta}$, with $\beta_1 = 1$ and $\alpha_1^2 = 0$; they are therefore the solutions by uniqueness. Hence $z_{\tau_t} \mid z \sim \mathcal N(\beta_{\tau_t} z, \alpha_{\tau_t}^2)$, which is the componentwise form of \cref{eq:app-noising}.
\end{proof}

\section{Experiment Details}\label{apx:experiment_details}
This appendix gives what is needed to reproduce the results in this paper. We present the DAWIS hyperparameters and how they were tuned, the backbone and training setup, the configuration of every baseline, and the data, splits and observation operators for SQG and \texttt{SEVIR}. DAWIS is tuned only on the filtering task, and the same settings are reused unchanged for both smoothers, so the smoothing results involve no additional tuning.

\subsection{DAWIS Hyperparameters and tuning}
We tune the hyperparameters only for the filtering experiments and use the same settings for the corresponding smoothing experiments. The complete set of hyperparameters is listed in \cref{tab:experiment_hyperparams}.
\begin{table}[H]
\centering
\caption{Hyperparameter configurations for DAWIS for all SQG and \texttt{SEVIR} experiments.}
\label{tab:experiment_hyperparams}
\begin{tabular}{lcccccc}
\toprule
\textbf{Hyperparameter} 
& \texttt{Noisy} 
& \texttt{Sparse} 
& \texttt{Multimodal} 
& \texttt{Saturating} 
& \texttt{SEVIR} \\
\midrule
Guide Method
& MMPS & MMPS & MMPS & MMPS & MMPS \\

Solver Steps
& 100 & 100 & 100 & 100 & 100 \\

$\boldsymbol{\varepsilon}$
& 0.03 & 0.03 & 0.03 & 0.03 & 2.5/1.0 \\

$\boldsymbol{\taumin^{\mathrm{start}}}$
& 0.4 & 0.4 & 0.4 & 0.4 & 0.4 \\

$\boldsymbol{\taumin^{\mathrm{end}}}$
& 0.3 & 0.1 & 0.2 & 0.4 & 0.1 \\

$\boldsymbol{\taumin^{\mathrm{init}}}$
& 0.0 & 0.0 & 0.0 & 0.0 & 1.0 \\

$\boldsymbol{\taumin^{\mathrm{init-IC}}}$
& 0.9 & 0.9 & 0.9 & 0.9 & 0.9 \\

Guidance Strength $\zeta$
& 1 & 1 & 1 & 20 & 2.5 \\

Guide all steps
& False & True & False & True & True \\
\bottomrule
\end{tabular}
\end{table}

\paragraph{Block smoothers:} The Block Gibbs Smoother, DAWIS Block Smoother and DAWIS Soft Block Smoother share the DAWIS window model and the same sweep. Each is initialized from the DAWIS Lagged Smoother on SQG and the DAWIS-Joint Lagged Smoother on \texttt{SEVIR}. Blocks of $w+1=7$ states slide over the trajectory with stride 3 and alternate between forward and backward sweeps, with 2 sweeps on SQG and 4 on \texttt{SEVIR}. 
For SQG, the two boundary states of each block condition the update. 
For a non-Markovian system such as SEVIR, the assumptions of \cref{eq:ssm} break and the block smoother should condition on more than two points.
We tried conditioning on two points for both DAWIS block smoother and Block Gibbs Smoother, but did not see an improvement.

The noising step is 
a sample from the interpolant \cref{eq:block-noising}
followed by 100 guided solver steps with MMPS and $\varepsilon=0.03$. The three methods differ only in their turning points (\cref{tab:hyperparams_block}). The Block Gibbs Smoother inverts the interior to pure noise ($\taumin=0$). The DAWIS Block Smoother uses a tuned interior turning point and keeps the boundary fixed. 
The soft variant also inverts the boundary, to $\taumin^{\partial\mathcal B}=0.9$. The guidance strength is $\zeta=1$ except on \texttt{Saturating}, where $\zeta=150$.
\begin{table}[H]
\centering
\caption{Interior turning point $\taumin$ of the block smoothers. The boundary turning point is $1$ for the Block Gibbs and DAWIS Block Smoothers and $0.9$ for the DAWIS Soft Block Smoother.}
\label{tab:hyperparams_block}
\begin{tabular}{lccccc}
\toprule
\textbf{Method} & \texttt{Noisy} & \texttt{Sparse} & \texttt{Multimodal} & \texttt{Saturating} & \texttt{SEVIR} \\
\midrule
Block Gibbs Smoother       & 0   & 0   & 0   & 0   & 0   \\
DAWIS Block Smoother       & 0.9 & 0.7 & 0.7 & 0.9 & 0.9 \\
DAWIS Soft Block Smoother  & 0.9 & 0.7 & 0.7 & 0.9 & 0.9 \\
\bottomrule
\end{tabular}
\end{table}

\subsection{Backbone Architecture}
\label{app:training}
We parameterize $\eta_{0,1}$ using a U-Net backbone \citep{UNET_Original}, with separate output channels for the two components. Preliminary experiments found this to perform comparably to using separate networks for $\eta_0$ and $\eta_1$ or layer-norm conditioning, while requiring less computation and memory. The backbone follows the general architecture used in \citet{song2020score} and \citet{Karras2022edm}, with the addition of circular padding to respect the periodic boundary conditions of the SQG data. It has three hierarchical levels, with attention at the second level for SQG and only at the deepest level for \texttt{SEVIR} due to its higher spatial resolution. Since the model processes a full window of $W$ states jointly, we scale its capacity with the window size by setting the hidden dimension to $32\times W$ and reduce the batch size accordingly to keep memory usage approximately constant. Unless stated otherwise, all window-based methods use $w+1=7$ states, so the DAWIS Lagged Smoother has a lag of $\ell=6$ steps (\SI{18}{h} on SQG, \SI{60}{min} on \texttt{SEVIR}). We do not investigate more specialized architectures, but DAWIS is compatible with any model that can be trained using a flow-matching objective.

All DAWIS models use the same backbone architecture and are implemented in PyTorch. Models are trained for 50 epochs on 8 A100 GPUs, using the same optimizer settings across datasets, as listed in \cref{tab:optimizer}.

\begin{table}[H]
\centering
\caption{Optimizer hyperparameters, shared across all datasets and window sizes.}
\label{tab:optimizer}
\begin{tabular}{lc}
\toprule
\textbf{Optimizer hyperparameters} & \\
\midrule
 Optimiser & AdamW \citep{loshchilov_fixing_2017} \\
 Initialization & Xavier Uniform \citep{glorot_understanding_2010} \\
 LR decay schedule & Cosine \citep{loshchilov_sgdr_2017} \\
 Peak LR & 1e-3 \\
 Minimum LR & 1e-5 \\
 Weight decay & 1e-2 \\
 Epochs & 50 \\
\bottomrule
\end{tabular}
\end{table}

\subsection{Baselines Details}
Here we provide additional details on the baseline methods, including their implementation and hyperparameter settings. Where a baseline was already tuned for this setup in prior work, we reuse those settings. The SDA-based baselines share the DAWIS backbone and are tuned per experiment, so the comparison with them separates the effect of the inference procedure from that of the learned prior.

\paragraph{LETKF:} We use the hyperparameters listed in \cref{tab:LETKF_hyperparams}, following the values used in \citep{andrae2026daisi}, which were already tuned for the experimental settings considered in this paper.

\begin{table}[H]
\centering
\caption{Hyperparameter configurations for LETKF.}
\label{tab:LETKF_hyperparams}
\begin{tabular}{lcccccccc}
\toprule
\textbf{Hyperparameter}
& \texttt{Noisy}
& \texttt{Sparse}
& \texttt{Multimodal}
& \texttt{Saturating}
& \texttt{SEVIR} \\
\midrule
Localization scale
& \SI{2500}{\kilo\meter}
& \SI{1500}{\kilo\meter}
& \SI{1500}{\kilo\meter}
& \SI{4500}{\kilo\meter}
& \SI{3}{px} \\

Relaxation to prior spread
& 0.5
& 0.5
& 0.5
& 0.9
& 0.0 \\
\bottomrule
\end{tabular}
\end{table}

\paragraph{FlowDAS:} For \texttt{SEVIR}, we use the pretrained backbone and original hyperparameters from \citet{chen2025flowdas}. For SQG, we use the backbone and hyperparameters from \citet{andrae2026daisi}, which were tuned for this experimental setting. All hyperparameters for all experiments for FlowDAS are shown in \cref{tab:hyperparams_flowdas}.
\begin{table}[H]
\centering
\caption{Hyperparameter configurations for FlowDAS for all SQG and \texttt{SEVIR} experiments.}
\label{tab:hyperparams_flowdas}
\begin{tabular}{lcccc}
\toprule
\textbf{Experiment} & \textbf{Solver Steps} & \textbf{MC Samples} & \textbf{Window} &
\textbf{Guidance Strength} $\zeta$ \\
\midrule
\texttt{Noisy}        & 500 & 25 & 6 & 1   \\
\texttt{Sparse}       & 500 & 25 & 6 & 3   \\
\texttt{Multimodal}   & 500 & 25 & 6 & 0.5 \\
\texttt{Saturating}   & 500 & 25 & 6 & 5   \\
\texttt{SEVIR}        & 500 & 25  & 6  & 0.1  \\
\bottomrule
\end{tabular}
\end{table}

\paragraph{EnSF:} We use the same implementation and hyperparameters as in \citet{andrae2026daisi} shown in \cref{tab:hyperparams_ensf}.
\begin{table}[H]
\centering
\caption{Hyperparameter configurations for EnSF for all SQG and \texttt{SEVIR} experiments.}
\label{tab:hyperparams_ensf}
\begin{tabular}{lcc}
\toprule
\textbf{Experiment} & \textbf{Solver Steps} & $\boldsymbol{\epsilon_{\alpha}}$ \\
\midrule
\texttt{Noisy}        & 1000 & 0.05 \\
\texttt{Sparse}       & 1000 & 0.05 \\
\texttt{Multimodal}   & 1000 & 0.05 \\
\texttt{Saturating}   & 1000 & 0.05 \\
\texttt{SEVIR}        & 1000 & 0.05 \\
\bottomrule
\end{tabular}
\end{table}

\paragraph{SDA Filter:} We use the same backbone that we use for DAWIS and we tune the hyperparameters for each experiment. The final hyperparameters are shown in \cref{tab:hyperparams_sda_filter}.
\begin{table}[H]
\centering
\caption{Hyperparameter configurations for SDA Filter for all SQG and \texttt{SEVIR} experiments.}
\label{tab:hyperparams_sda_filter}
\begin{tabular}{lcccccccl}
\toprule
\textbf{Experiment} & \textbf{Guide Method} & \textbf{Solver Steps} &
$\boldsymbol{\varepsilon}$ &
\textbf{Corrections} & $\boldsymbol{\tau}$ &
\textbf{Guidance Strength} $\zeta$ \\
\midrule
\texttt{Noisy}        & MMPS & 100 & 1 & 1 & 0.3 & 2  \\
\texttt{Sparse}       & MMPS & 100 & 1 & 1 & 0.3 & 2  \\
\texttt{Multimodal}   & MMPS & 100 & 1 & 1 & 0.5 & 2  \\
\texttt{Saturating}   & MMPS & 100 & 1 & 1 & 0.3 & 20 \\
\texttt{SEVIR}        & MMPS & 100 & 1 & 1 & 0.3 & 10 \\
\bottomrule
\end{tabular}
\end{table}

\paragraph{DAISI:} We use the same hyperparameters (see \cref{tab:hyperparams_daisi}) and pretrained backbone as in the original DAISI implementation \citep{andrae2026daisi}.
\begin{table}[H]
\centering
\caption{Hyperparameter configurations for DAISI for all SQG and \texttt{SEVIR} experiments.}
\label{tab:hyperparams_daisi}
\begin{tabular}{lccccc}
\toprule
\textbf{Experiment} & \textbf{Guide Method} & \textbf{Solver Steps} &
$\boldsymbol{\varepsilon}$ & $\boldsymbol{t_{\min}}$ &
\textbf{Guidance Strength} $\zeta$ \\
\midrule
\texttt{Noisy}        & DPS  & 100 & 0.03 & 0.4 & 10 \\
\texttt{Sparse}       & MMPS & 100 & 0.03 & 0.3 & 1  \\
\texttt{Multimodal}   & MMPS & 100 & 0.03 & 0.3 & 1  \\
\texttt{Saturating}   & MMPS & 100 & 0.03 & 0.3 & 20 \\
\texttt{SEVIR}        & MMPS & 100 & 0.03 & 0.3 & 1 \\
\bottomrule
\end{tabular}
\end{table}

\paragraph{Joint AR:} Both variants use the same window model as DAWIS and no separate forecast model. Each cycle generates part of the window from pure noise, conditioned on the rest of the window held at data ($\tau=1$) \citep{shysheya2024on}. Joint AR 1$|$W holds the $w=6$ most recent analyses fixed and generates only the new state, guided by its observation. Joint AR W$|$1 holds only the oldest state fixed and regenerates the other six, guided by every observation in the window, so it also revises past states. Neither variant has an inversion step. The hyperparameters, tuned per experiment, are shown in \cref{tab:hyperparams_jointar}.
\begin{table}[H]
\centering
\caption{Hyperparameter configurations for Joint AR for all SQG and \texttt{SEVIR} experiments. Both variants use MMPS and $\varepsilon=1$.}
\label{tab:hyperparams_jointar}
\begin{tabular}{lcccc}
\toprule
& \multicolumn{2}{c}{\textbf{Joint AR 1$|$W}} & \multicolumn{2}{c}{\textbf{Joint AR W$|$1}} \\
\cmidrule(lr){2-3}\cmidrule(lr){4-5}
\textbf{Experiment} & \textbf{Solver Steps} & $\zeta$ & \textbf{Solver Steps} & $\zeta$ \\
\midrule
\texttt{Noisy}        & 100 & 5  & 100 & 4   \\
\texttt{Sparse}       & 100 & 5  & 100 & 7   \\
\texttt{Multimodal}   & 100 & 5  & 25  & 5   \\
\texttt{Saturating}   & 100 & 20 & 100 & 150 \\
\texttt{SEVIR}        & 100 & 3  & 100 & 1   \\
\bottomrule
\end{tabular}
\end{table}

\paragraph{ForcingDAS-Pyr:} We implement the pyramid schedule of \citet{jia2026forcingdasunifiedrobustdata} on the DAWIS window model, so both methods share the same prior. The window of $T=w+1=7$ states sits on a noise ladder: at the start of each cycle, state $j$ is at flow time $\tau^j = (T-1-j)/(T-1)$ for $j \ge 1$. Each cycle advances every state one rung of $1/(T-1)$ under guidance from all observations in the window. The first state is held at data as clean context ($n_{\text{fixed}}=1$), so each state is finalized after $w-1=5$ future observations. All experiments use MMPS, 20 solver steps per cycle and $\varepsilon=0.03$. The guidance strength is tuned per experiment: $\zeta = 2, 15, 4, 150$ and $3$ for \texttt{Noisy}, \texttt{Sparse}, \texttt{Multimodal}, \texttt{Saturating} and \texttt{SEVIR}.

\paragraph{SDA:} We use the same backbone that we use for DAWIS and we tune the hyperparameters for each experiment. The final hyperparameters are shown in \cref{tab:hyperparams_sda}.
\begin{table}[H]
\centering
\caption{Hyperparameter configurations for SDA for all SQG and \texttt{SEVIR} experiments.}
\label{tab:hyperparams_sda}
\begin{tabular}{lcccccc}
\toprule
\textbf{Experiment} & \textbf{Guide Method} & \textbf{Solver Steps} &
$\boldsymbol{\varepsilon}$ & \textbf{Corrections} & $\boldsymbol{\tau}$ &
\textbf{Guidance Strength} $\zeta$ \\
\midrule
\texttt{Noisy}        & MMPS & 100 & 1 & 1 & 0.5 & 2  \\
\texttt{Sparse}       & MMPS & 100 & 1 & 1 & 0.5 & 5  \\
\texttt{Multimodal}   & MMPS & 200 & 1 & 3 & 1.0 & 3  \\
\texttt{Saturating}   & MMPS & 100 & 1 & 2 & 0.5 & 30 \\
\texttt{SEVIR}        & MMPS & 100 & 1 & 1 & 0.5   & 3  \\
\bottomrule
\end{tabular}
\end{table}

\subsection{Surface Quasi-Geostrophic (SQG)}
For the SQG experiments, we reuse the training dataset from \citet{andrae2026daisi}, consisting of 2,000 trajectories of 100 states sampled at 3-hour intervals. Each trajectory is initialized from a random state and spun up for 300 days to reach approximate stationarity before the recorded trajectory begins. Evaluation is performed on 10 held-out trajectories of the same length, with metrics averaged over the final 20 states. Since the data are approximately mean-zero, we normalize each channel only by its standard deviation. The experimental settings are summarized in \cref{tab:experiments}, and the corresponding observation operators are visualized in \cref{fig:obs}.
\setlength{\tabcolsep}{4pt}
\begin{table}[h]
\centering
\small
\caption{Experiment configurations.}
\label{tab:experiments}
\begin{tabular}{c c c c c}
\toprule
\textbf{Experiment} & \textbf{Data} & \textbf{Obs operator}& $\sigma_{\text{obs}}$ & \textbf{Sparsity} \\
\midrule
\texttt{Noisy} & SQG-$64$  & $x$& $5$& $25\%$\\
\texttt{Sparse} & SQG-$64$& $x$&  $1$& $5\%$\\
\texttt{Multimodal} & SQG-$64$  & $(x/7)^2$&  $1$& $25\%$\\
\texttt{Saturating} & SQG-$64$  & $\arctan(x)$&  $0.01$& $25\%$\\
\texttt{SEVIR} & \texttt{SEVIR}  & $x$&  $0.001$& $10\%$\\
\bottomrule
\end{tabular}
\end{table}

\begin{figure}[h]
  \centering
  \includegraphics[width=\linewidth]{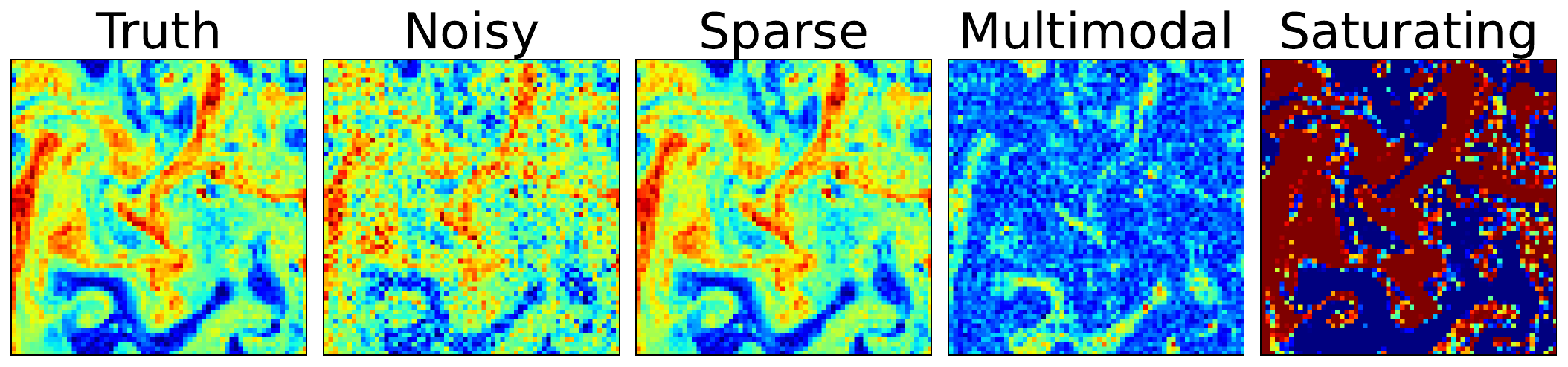}
  \caption{Visualization of the observations for each configuration before sparsity is applied.}
  \label{fig:obs}
\end{figure}

\subsection{SEVIR}
\textbf{Data and split.} We use the \texttt{SEVIR} VIL low-resolution product ($128\times128$, 25 frames per trajectory at \SI{10}{min} intervals). Following \citet{chen2025flowdas}, events after 2019-06-01 form the test set; earlier events are split into training and validation by a seeded random \SI{10}{\percent} holdout. This gives 13{,}434 / 1{,}492 / 4{,}053 train / validation / test events. We report results on 10 test trajectories chosen randomly from the test set.

\textbf{Normalization.} All networks trained here use the FlowDAS normalization $(\mathrm{VIL}/255-0.5)/0.1$ except for DAISI's pretrained prior which was trained on $\mathrm{VIL}/255$. All metrics are computed in raw VIL units ($0$--$255$), so scores are comparable across methods whose networks use different normalizations.

\textbf{Initialization.} The first 6 frames of each trajectory are given as initialization without any noise following \citet{chen2025flowdas}, and the remaining 19 frames are assimilated. Filters are scored over the last 10 steps and smoothers over all 19.

\section{Additional results}\label{apx:additional_results}
This appendix extends the CRPS evaluation in \cref{tab:crps} with the RMSE (\cref{tab:rmse}) and the spread-skill ratio (SSR, \cref{tab:ssr}), and breaks the scores down by setting, over time, by calibration and by spectral content. The RMSE ranking closely follows the CRPS ranking. On SQG, DAWIS Filter, DAWIS Lagged Smoother and DAWIS Block Smoother are the best, or tied for best, in their category in every setting. On \texttt{SEVIR}, the DAWIS Block Smoother has the lowest RMSE overall. The SSR shows that the DAWIS filters stay close to calibration on SQG. SDA-Filter and EnSF are clearly underdispersive in most settings. The DAWIS lagged smoothers are mildly overdispersive on \texttt{Saturating}. On \texttt{SEVIR}, all DAWIS variants are overdispersive  and the SSR varies widely between trajectories.

\begin{table}[H]
\centering
\small
\caption{The RMSE for experiments on SQG and \texttt{SEVIR}.
We display the mean and standard deviation across 10 independent trajectories, averaged over the last 20 (10 for \texttt{SEVIR}) steps.
The best score for each experiment is highlighted in \textbf{bold} and the second best with an \underline{underline}.
Smoothing methods calculate the metrics over the whole trajectory.}
\label{tab:rmse}
\setlength{\tabcolsep}{4pt}
\begin{tabular}{l c c c c c}
\toprule
\textbf{Method}
& \textbf{Noisy}
& \textbf{Sparse}
& \textbf{Multimodal}
& \textbf{Saturating}
& \textbf{SEVIR} \\

\midrule
\multicolumn{6}{l}{\textit{Filtering}} \\
\midrule

EnSF \cite{bao2024ensemble}
& $7.80{\scriptstyle \pm1.07}$
& $6.95{\scriptstyle \pm0.82}$
& $3.31{\scriptstyle \pm1.95}$
& $8.80{\scriptstyle \pm2.39}$
& $35.12{\scriptstyle \pm26.68}$ \\

FlowDAS \cite{chen2025flowdas}
& $4.87{\scriptstyle \pm0.79}$
& $6.09{\scriptstyle \pm0.94}$
& $4.84{\scriptstyle \pm0.81}$
& $6.77{\scriptstyle \pm0.61}$
& $27.31{\scriptstyle \pm13.16}$ \\

LETKF \cite{liang2025ensemble}
& $2.49{\scriptstyle \pm0.26}$
& $4.31{\scriptstyle \pm0.47}$
& $8.97{\scriptstyle \pm2.04}$
& $4.19{\scriptstyle \pm1.86}$
& $12.19{\scriptstyle \pm7.04}$ \\

DAISI \cite{andrae2026daisi}
& $2.42{\scriptstyle \pm0.25}$
& $3.31{\scriptstyle \pm0.40}$
& $2.99{\scriptstyle \pm0.19}$
& $3.65{\scriptstyle \pm0.99}$
& $11.02{\scriptstyle \pm6.10}$ \\

\addlinespace[0.2em]
\hdashline
\addlinespace[0.2em]

SDA-Filter
& $2.36{\scriptstyle \pm0.22}$
& $3.16{\scriptstyle \pm0.40}$
& $3.69{\scriptstyle \pm0.57}$
& $5.98{\scriptstyle \pm0.77}$
& $\underline{10.05}{\scriptstyle \pm5.00}$ \\

Joint AR 1$|$W \cite{shysheya2024on}
& $6.89{\scriptstyle \pm0.67}$
& $7.37{\scriptstyle \pm0.90}$
& $7.22{\scriptstyle \pm0.70}$
& $6.86{\scriptstyle \pm0.78}$
& $16.19{\scriptstyle \pm7.84}$ \\

Joint AR W$|$1 \cite{shysheya2024on}
& $2.28{\scriptstyle \pm0.22}$
& $3.23{\scriptstyle \pm0.31}$
& $\underline{2.22}{\scriptstyle \pm0.16}$
& $2.23{\scriptstyle \pm0.27}$
& $11.10{\scriptstyle \pm6.50}$ \\

DAWIS Filter
& $\mathbf{1.91}{\scriptstyle \pm0.18}$
& $\mathbf{2.40}{\scriptstyle \pm0.28}$
& $\mathbf{1.96}{\scriptstyle \pm0.15}$
& $\mathbf{1.71}{\scriptstyle \pm0.23}$
& $10.14{\scriptstyle \pm5.08}$ \\

DAWIS-Joint Filter
& $\underline{2.20}{\scriptstyle \pm0.23}$
& $\underline{2.47}{\scriptstyle \pm0.30}$
& $2.29{\scriptstyle \pm0.24}$
& $\underline{1.93}{\scriptstyle \pm0.29}$
& $\mathbf{9.97}{\scriptstyle \pm5.42}$ \\

\midrule
\multicolumn{6}{l}{\textit{Fixed-lag Smoothing}} \\
\midrule

SDA-Filter
& $2.43{\scriptstyle \pm0.13}$
& $2.73{\scriptstyle \pm0.21}$
& $3.14{\scriptstyle \pm0.42}$
& $4.91{\scriptstyle \pm0.60}$
& $10.94{\scriptstyle \pm5.83}$ \\

ForcingDAS-Pyr \cite{jia2026forcingdasunifiedrobustdata}
& $2.55{\scriptstyle \pm0.14}$
& $2.53{\scriptstyle \pm0.21}$
& $2.03{\scriptstyle \pm0.17}$
& $2.12{\scriptstyle \pm0.27}$
& $9.52{\scriptstyle \pm5.94}$ \\

Joint AR W$|$1 \cite{shysheya2024on}
& $2.01{\scriptstyle \pm0.09}$
& $2.53{\scriptstyle \pm0.10}$
& $\underline{1.94}{\scriptstyle \pm0.08}$
& $1.98{\scriptstyle \pm0.15}$
& $9.73{\scriptstyle \pm6.28}$ \\

DAWIS Lagged Smoother
& $\mathbf{1.72}{\scriptstyle \pm0.08}$
& $\mathbf{2.02}{\scriptstyle \pm0.11}$
& $\mathbf{1.75}{\scriptstyle \pm0.07}$
& $\mathbf{1.48}{\scriptstyle \pm0.10}$
& $\mathbf{9.05}{\scriptstyle \pm5.24}$ \\

DAWIS-Joint Lagged Smoother
& $\underline{1.98}{\scriptstyle \pm0.11}$
& $\underline{2.08}{\scriptstyle \pm0.12}$
& $2.04{\scriptstyle \pm0.11}$
& $\underline{1.62}{\scriptstyle \pm0.13}$
& $\underline{9.09}{\scriptstyle \pm5.31}$ \\

\midrule
\multicolumn{6}{l}{\textit{Joint Smoothing}} \\
\midrule

SDA \cite{rozet2023score}
& $1.92{\scriptstyle \pm0.12}$
& $\underline{1.94}{\scriptstyle \pm0.11}$
& $\underline{5.63}{\scriptstyle \pm0.59}$
& $2.11{\scriptstyle \pm0.10}$
& $9.89{\scriptstyle \pm5.78}$ \\

Block Gibbs Smoother
& $\underline{1.70}{\scriptstyle \pm0.10}$
& $\mathbf{1.92}{\scriptstyle \pm0.12}$
& $\mathbf{1.68}{\scriptstyle \pm0.09}$
& $\underline{1.66}{\scriptstyle \pm0.15}$
& $\underline{9.27}{\scriptstyle \pm5.63}$ \\

DAWIS Block Smoother
& $\mathbf{1.67}{\scriptstyle \pm0.09}$
& $\mathbf{1.92}{\scriptstyle \pm0.12}$
& $\mathbf{1.68}{\scriptstyle \pm0.08}$
& $\mathbf{1.36}{\scriptstyle \pm0.11}$
& $\mathbf{8.93}{\scriptstyle \pm5.45}$ \\

\bottomrule
\end{tabular}
\end{table}

\begin{table}[H]
\centering
\small
\caption{The SSR for experiments on SQG and \texttt{SEVIR}.
We display the mean and standard deviation across 10 independent trajectories, averaged over the last 20 (10 for \texttt{SEVIR}) steps.
The best score for each experiment is highlighted in \textbf{bold} and the second best with an \underline{underline}.
Smoothing methods calculate the metrics over the whole trajectory.}
\label{tab:ssr}
\setlength{\tabcolsep}{4pt}
\begin{tabular}{l c c c c c}
\toprule
\textbf{Method}
& \textbf{Noisy}
& \textbf{Sparse}
& \textbf{Multimodal}
& \textbf{Saturating}
& \textbf{SEVIR} \\

\midrule
\multicolumn{6}{l}{\textit{Filtering}} \\
\midrule

EnSF \cite{bao2024ensemble}
& $0.43{\scriptstyle \pm0.05}$
& $0.49{\scriptstyle \pm0.05}$
& $1.24{\scriptstyle \pm0.33}$
& $0.48{\scriptstyle \pm0.39}$
& $0.35{\scriptstyle \pm0.21}$ \\

FlowDAS \cite{chen2025flowdas}
& $0.66{\scriptstyle \pm0.07}$
& $\underline{0.96}{\scriptstyle \pm0.06}$
& $0.87{\scriptstyle \pm0.11}$
& $\mathbf{1.00}{\scriptstyle \pm0.09}$
& $0.53{\scriptstyle \pm0.23}$ \\

LETKF \cite{liang2025ensemble}
& $\mathbf{1.00}{\scriptstyle \pm0.05}$
& $1.05{\scriptstyle \pm0.07}$
& $0.78{\scriptstyle \pm0.07}$
& $0.63{\scriptstyle \pm0.12}$
& $1.16{\scriptstyle \pm0.40}$ \\

DAISI \cite{andrae2026daisi}
& $0.89{\scriptstyle \pm0.03}$
& $\underline{1.04}{\scriptstyle \pm0.04}$
& $1.04{\scriptstyle \pm0.19}$
& $1.34{\scriptstyle \pm0.18}$
& $\underline{0.92}{\scriptstyle \pm0.19}$ \\

\addlinespace[0.2em]
\hdashline
\addlinespace[0.2em]

SDA-Filter
& $0.72{\scriptstyle \pm0.02}$
& $0.60{\scriptstyle \pm0.06}$
& $0.64{\scriptstyle \pm0.06}$
& $0.43{\scriptstyle \pm0.07}$
& $1.34{\scriptstyle \pm0.82}$ \\

Joint AR 1$|$W \cite{shysheya2024on}
& $1.09{\scriptstyle \pm0.07}$
& $1.09{\scriptstyle \pm0.11}$
& $\underline{1.03}{\scriptstyle \pm0.07}$
& $1.28{\scriptstyle \pm0.10}$
& $\underline{0.92}{\scriptstyle \pm0.08}$ \\

Joint AR W$|$1 \cite{shysheya2024on}
& $0.74{\scriptstyle \pm0.04}$
& $\mathbf{0.97}{\scriptstyle \pm0.03}$
& $0.86{\scriptstyle \pm0.04}$
& $0.66{\scriptstyle \pm0.04}$
& $\mathbf{0.95}{\scriptstyle \pm0.34}$ \\

DAWIS Filter
& $0.89{\scriptstyle \pm0.01}$
& $1.09{\scriptstyle \pm0.06}$
& $0.96{\scriptstyle \pm0.01}$
& $1.11{\scriptstyle \pm0.04}$
& $1.29{\scriptstyle \pm0.81}$ \\

DAWIS-Joint Filter
& $\underline{0.92}{\scriptstyle \pm0.03}$
& $1.07{\scriptstyle \pm0.05}$
& $\mathbf{0.98}{\scriptstyle \pm0.03}$
& $\underline{1.07}{\scriptstyle \pm0.03}$
& $1.36{\scriptstyle \pm0.89}$ \\

\midrule
\multicolumn{6}{l}{\textit{Fixed-lag Smoothing}} \\
\midrule

SDA-Filter
& $0.74{\scriptstyle \pm0.01}$
& $0.61{\scriptstyle \pm0.05}$
& $0.64{\scriptstyle \pm0.06}$
& $0.51{\scriptstyle \pm0.08}$
& $1.16{\scriptstyle \pm0.63}$ \\

ForcingDAS-Pyr \cite{jia2026forcingdasunifiedrobustdata}
& $0.90{\scriptstyle \pm0.02}$
& $\underline{1.02}{\scriptstyle \pm0.02}$
& $\mathbf{0.97}{\scriptstyle \pm0.03}$
& $\mathbf{1.02}{\scriptstyle \pm0.02}$
& $\mathbf{0.97}{\scriptstyle \pm0.16}$ \\

Joint AR W$|$1 \cite{shysheya2024on}
& $0.77{\scriptstyle \pm0.04}$
& $\mathbf{0.99}{\scriptstyle \pm0.02}$
& $0.89{\scriptstyle \pm0.04}$
& $0.71{\scriptstyle \pm0.02}$
& $\underline{1.06}{\scriptstyle \pm0.50}$ \\

DAWIS Lagged Smoother
& $\underline{0.99}{\scriptstyle \pm0.01}$
& $1.15{\scriptstyle \pm0.03}$
& $\underline{1.06}{\scriptstyle \pm0.01}$
& $1.21{\scriptstyle \pm0.02}$
& $1.25{\scriptstyle \pm0.48}$ \\

DAWIS-Joint Lagged Smoother
& $\mathbf{1.00}{\scriptstyle \pm0.02}$
& $1.15{\scriptstyle \pm0.03}$
& $\underline{1.06}{\scriptstyle \pm0.02}$
& $\underline{1.16}{\scriptstyle \pm0.02}$
& $1.27{\scriptstyle \pm0.50}$ \\

\midrule
\multicolumn{6}{l}{\textit{Joint Smoothing}} \\
\midrule

SDA \cite{rozet2023score}
& $0.69{\scriptstyle \pm0.01}$
& $0.73{\scriptstyle \pm0.01}$
& $0.93{\scriptstyle \pm0.07}$
& $0.63{\scriptstyle \pm0.03}$
& $1.48{\scriptstyle \pm0.99}$ \\

Block Gibbs Smoother
& $\underline{0.89}{\scriptstyle \pm0.01}$
& $\mathbf{1.03}{\scriptstyle \pm0.04}$
& $\underline{0.94}{\scriptstyle \pm0.01}$
& $\mathbf{1.00}{\scriptstyle \pm0.02}$
& $\mathbf{1.16}{\scriptstyle \pm0.59}$ \\

DAWIS Block Smoother
& $\mathbf{0.92}{\scriptstyle \pm0.01}$
& $\underline{1.06}{\scriptstyle \pm0.04}$
& $\mathbf{0.95}{\scriptstyle \pm0.01}$
& $\underline{1.12}{\scriptstyle \pm0.02}$
& $\underline{1.31}{\scriptstyle \pm0.67}$ \\

\bottomrule
\end{tabular}
\end{table}

\subsection{Scorecards}
\Cref{fig:scorecard-crps,fig:scorecard-rmse} show the relative improvement of each method over a reference. The reference is DAWIS Filter for filtering and DAWIS Lagged Smoother for both smoothing categories, so a negative entry means the method is worse than DAWIS.  First, the margins on SQG are large: the strongest filtering baseline, Joint AR W$|$1, is still $16$--$39\%$ worse than DAWIS Filter in CRPS. On \texttt{SEVIR} the gaps close, and SDA-Filter and DAWIS-Joint Filter are slightly better than the reference. Second, the joint-smoothing rows show what the block sweep adds on top of its initialization. The DAWIS Block Smoother improves by $3$--$11\%$, most on \texttt{Saturating}. The Block Gibbs Smoother, which regenerates from pure noise, makes the initialization worse by $14\%$ on \texttt{Saturating} and $6\%$ on \texttt{SEVIR}.

\begin{figure}[H]
  \centering
  \includegraphics[width=\linewidth]{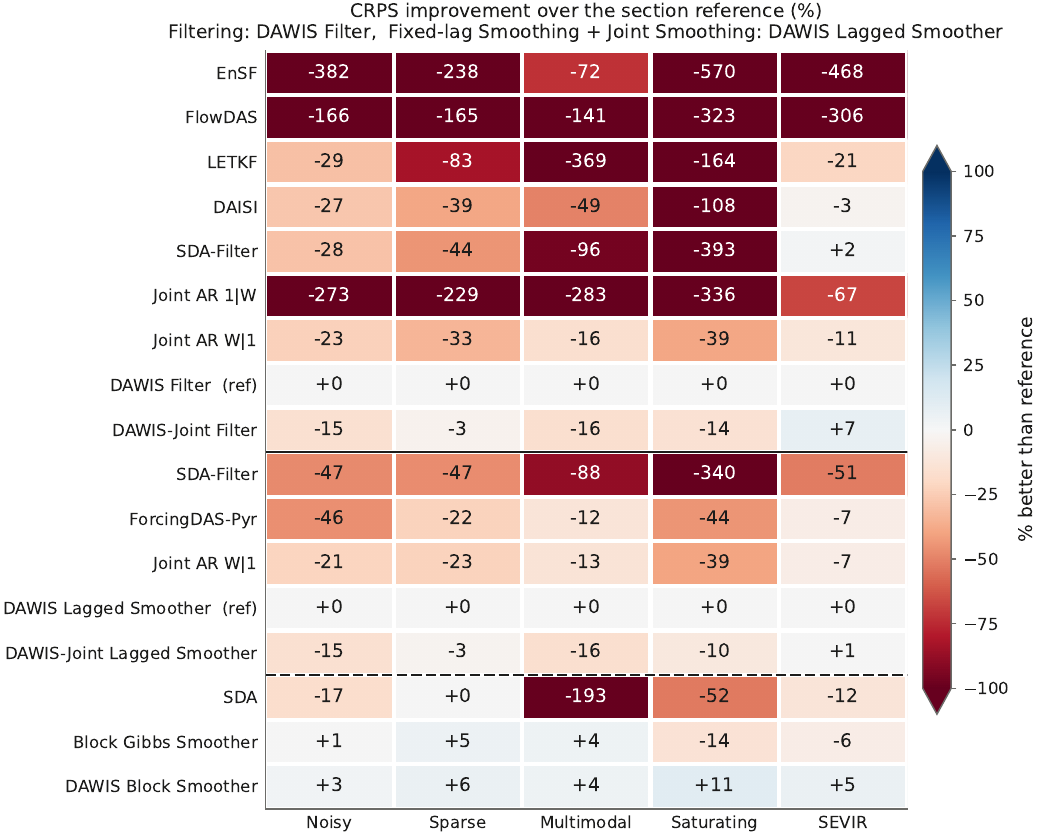}
  \caption{%
    CRPS scorecard against DAWIS.
  }
  \label{fig:scorecard-crps}
\end{figure}

\begin{figure}[H]
  \centering
  \includegraphics[width=\linewidth]{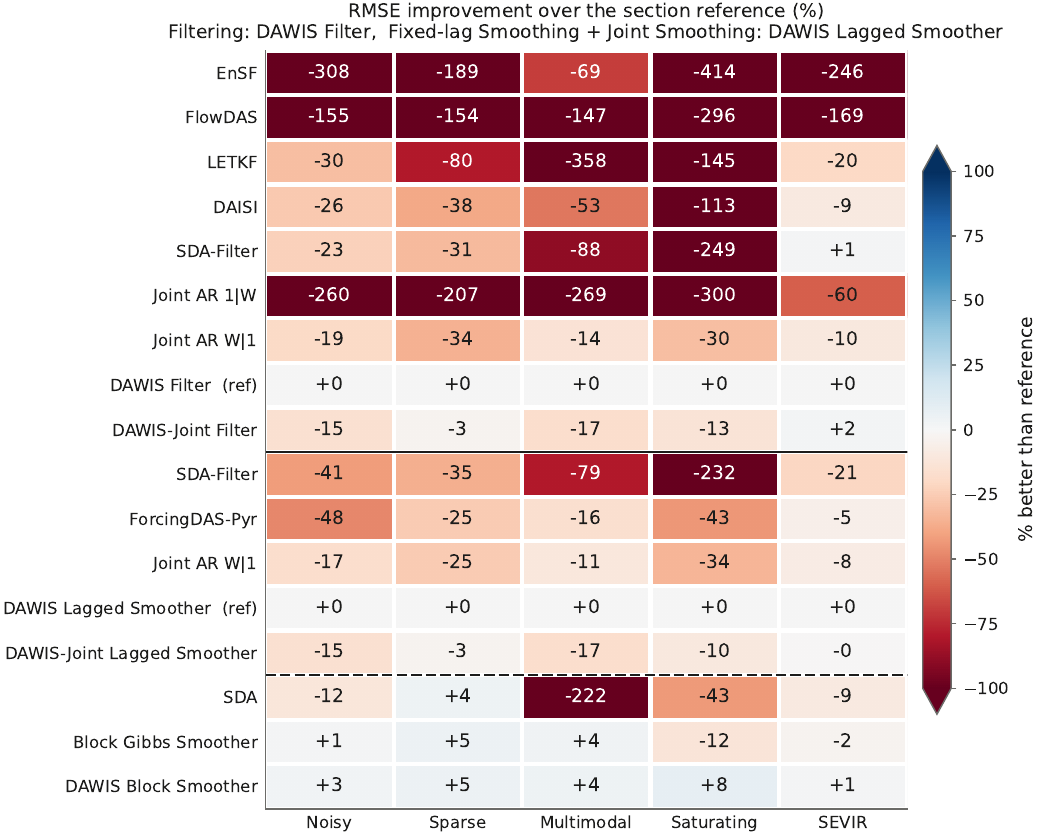}
  \caption{%
    RMSE scorecard against DAWIS.
  }
  \label{fig:scorecard-rmse}
\end{figure}

\subsection{Scores over time -- filtering}
\Cref{fig:crps-time-filtering,fig:rmse-time-filtering} show the analysis error at every assimilation step, averaged over the 10 trajectories. After a spin-up of roughly $50$\,h, both DAWIS filters settle at a constant error level and stay there for the rest of the trajectory. This supports averaging over the last 20 steps in \cref{tab:crps}. Several baselines instead drift. EnSF, FlowDAS and Joint AR 1$|$W accumulate error throughout the trajectory in every SQG setting. LETKF diverges within the first $\sim$50\,h on \texttt{Multimodal}. On \texttt{Saturating}, LETKF and DAISI start to degrade late in the trajectory.

\begin{figure}[H]
  \centering
  \includegraphics[width=\linewidth]{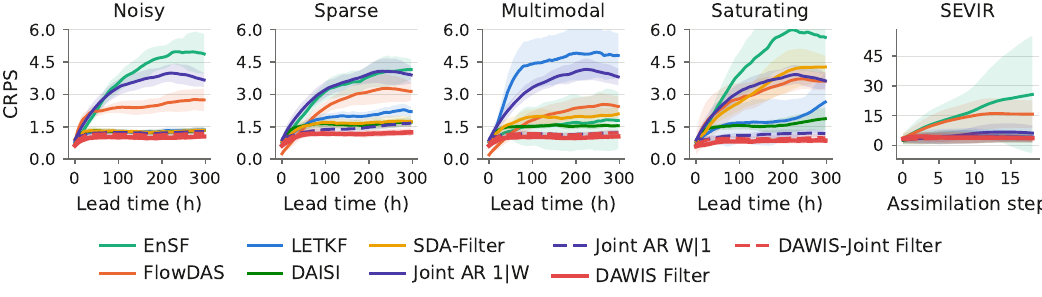}
  \caption{%
    The CRPS over the assimilation window for the filtering
    methods.
  }
  \label{fig:crps-time-filtering}
\end{figure}

\begin{figure}[H]
  \centering
  \includegraphics[width=\linewidth]{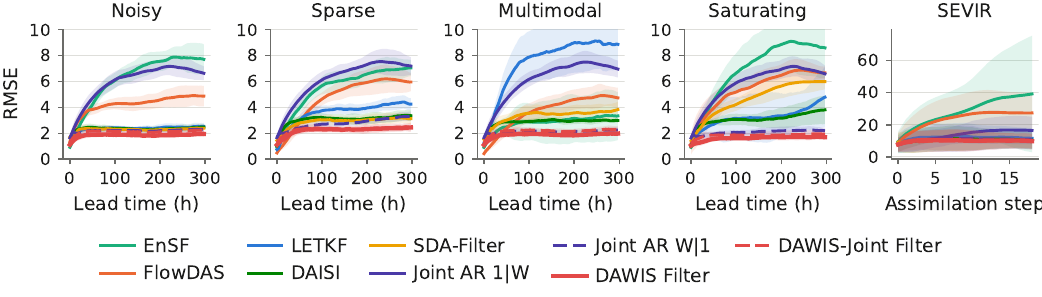}
  \caption{%
    The RMSE over the assimilation window for the filtering
    methods.
  }
  \label{fig:rmse-time-filtering}
\end{figure}

\subsection{Scores over time -- smoothing}
\Cref{fig:crps-time-smoothing,fig:rmse-time-smoothing} show the error at every assimilation step, averaged over the 10 trajectories for the smoothers. The DAWIS lagged smoothers have flat error over the whole trajectory. All fixed-lag smoothers show an increase in error over the last few steps, where fewer than $\ell$ future observations are available and the smoother reduces to a filter. \Cref{fig:filter-vs-smoother} isolates the gain from look-ahead within a single DAWIS run. After spin-up, the lagged estimate is below the filter estimate at every step, and the gap closes at the final step. This gain comes at no extra cost, because both estimates come from the same cycle.

\begin{figure}[H]
  \centering
  \includegraphics[width=\linewidth]{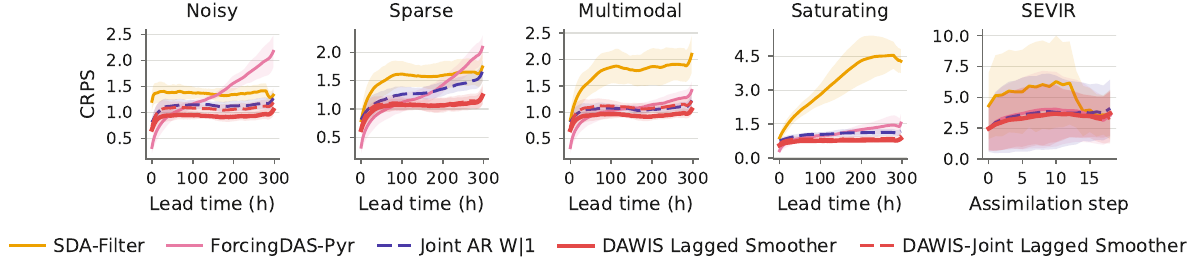}
  \caption{%
    The CRPS over the assimilation window for all smoothing
    methods.
  }
  \label{fig:crps-time-smoothing}
\end{figure}

\begin{figure}[H]
  \centering
  \includegraphics[width=\linewidth]{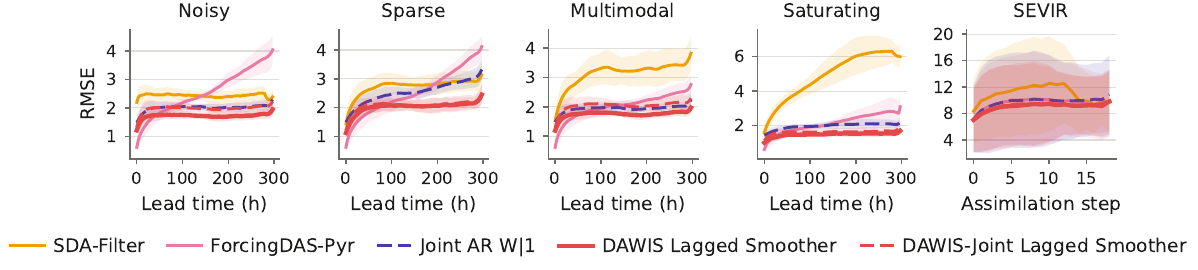}
  \caption{%
    The RMSE over the assimilation window for all smoothing
    methods.
  }
  \label{fig:rmse-time-smoothing}
\end{figure}

\begin{figure}[H]
  \centering
  \includegraphics[width=\linewidth]{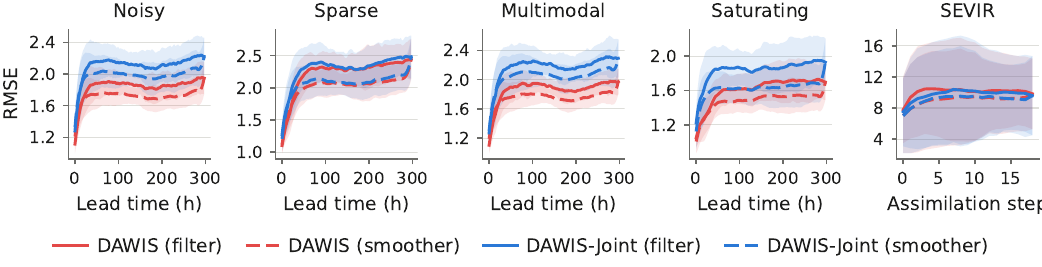}
  \caption{%
    The improvement in RMSE gained by using the lagged smoother for the DAWIS experiments.
  }
  \label{fig:filter-vs-smoother}
\end{figure}

\subsection{Calibration}
We assess calibration with the spread-skill ratio over time (\cref{fig:ssr-time-filtering}) and with verification-rank histograms. A calibrated ensemble has SSR $\approx 1$ and a flat histogram. SSR $<1$ with a U-shaped histogram indicates underdispersion, and SSR $>1$ with a dome-shaped histogram indicates overdispersion. All ensembles start overdispersive. After this transient, DAWIS stays close to one for the whole trajectory in every SQG setting. SDA-Filter and EnSF settle at $0.4$--$0.7$, and DAISI and Joint AR 1$|$W drift towards overdispersion on \texttt{Saturating}. The rank histograms agree. The DAWIS histograms are nearly flat, with only slightly raised end bins. On \texttt{Multimodal}, LETKF, FlowDAS, Joint AR W$|$1 and especially SDA-Filter are U-shaped, and EnSF is dome-shaped.

\begin{figure}[H]
  \centering
  \includegraphics[width=\linewidth]{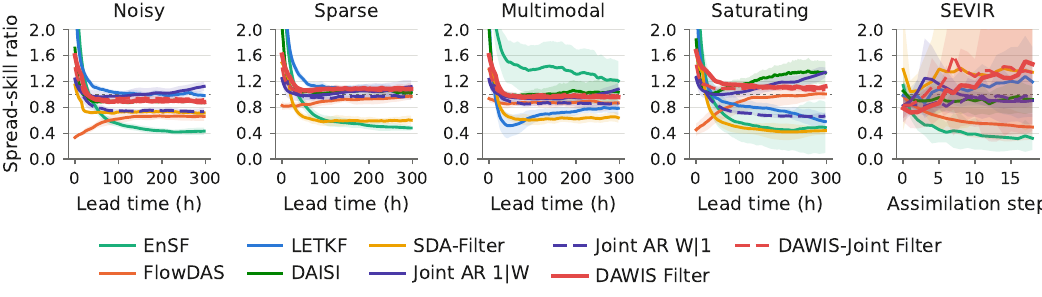}
  \caption{%
    Spread-skill ratio of the filtering methods against its ideal
    value of one.
  }
  \label{fig:ssr-time-filtering}
\end{figure}

\begin{figure}[H]
  \centering
  \includegraphics[width=\linewidth]{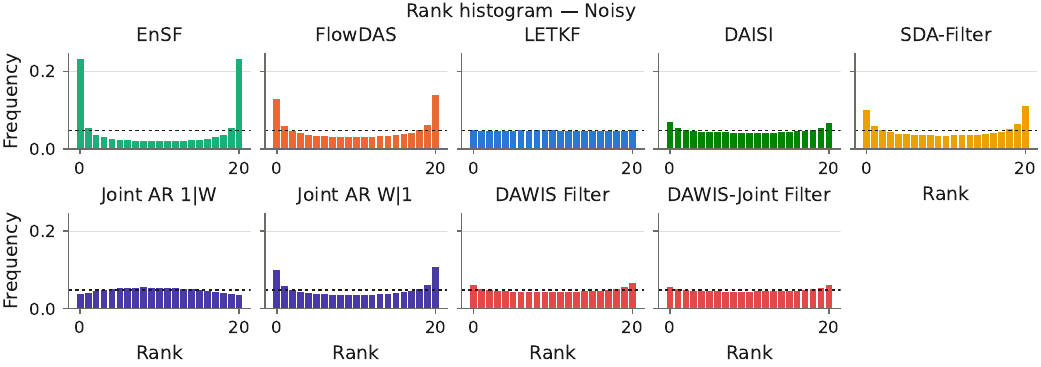}
  \caption{%
    Rank histograms of the filtering methods for \texttt{Noisy}.
  }
  \label{fig:rank-filtering-noisy}
\end{figure}

\begin{figure}[H]
  \centering
  \includegraphics[width=\linewidth]{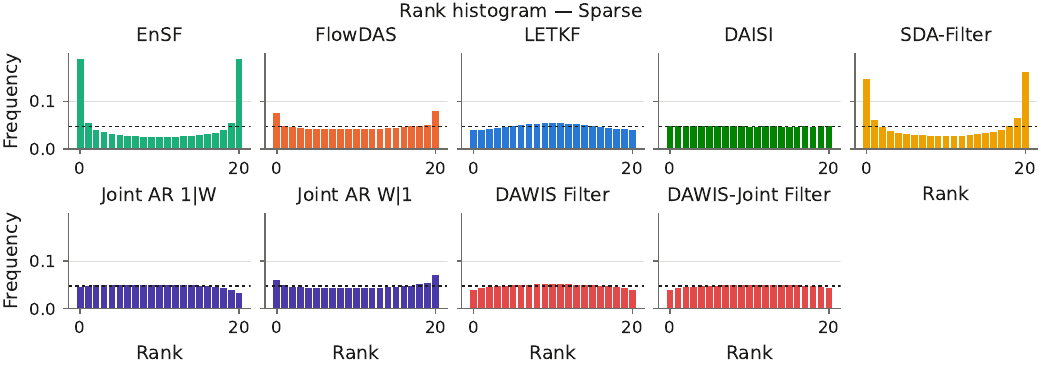}
  \caption{%
    Rank histograms of the filtering methods for \texttt{Sparse}.
  }
  \label{fig:rank-filtering-sparse}
\end{figure}

\begin{figure}[H]
  \centering
  \includegraphics[width=\linewidth]{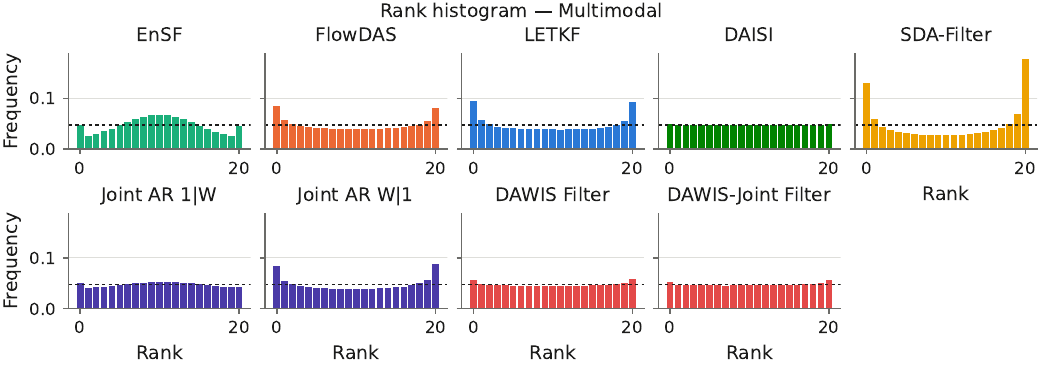}
  \caption{%
    Rank histograms of the filtering methods for \texttt{Multimodal}.
  }
  \label{fig:rank-filtering-multimodal}
\end{figure}

\begin{figure}[H]
  \centering
  \includegraphics[width=\linewidth]{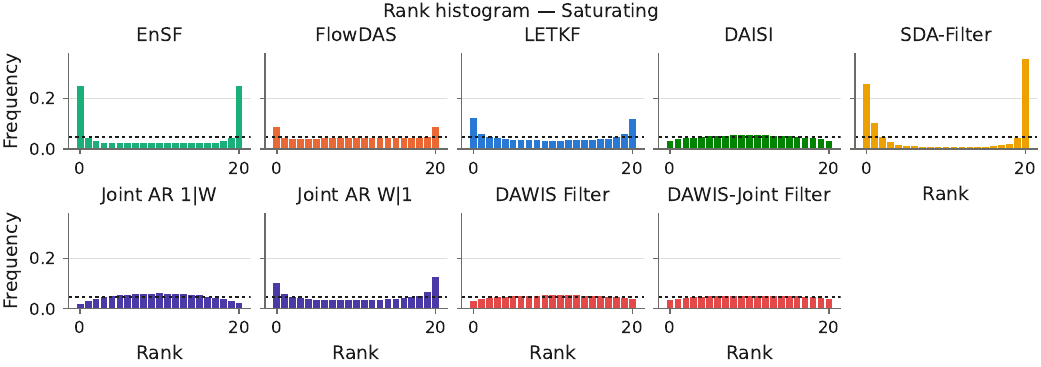}
  \caption{%
    Rank histograms of the filtering methods for \texttt{Saturating}.
  }
  \label{fig:rank-filtering-saturating}
\end{figure}

\begin{figure}[H]
  \centering
  \includegraphics[width=\linewidth]{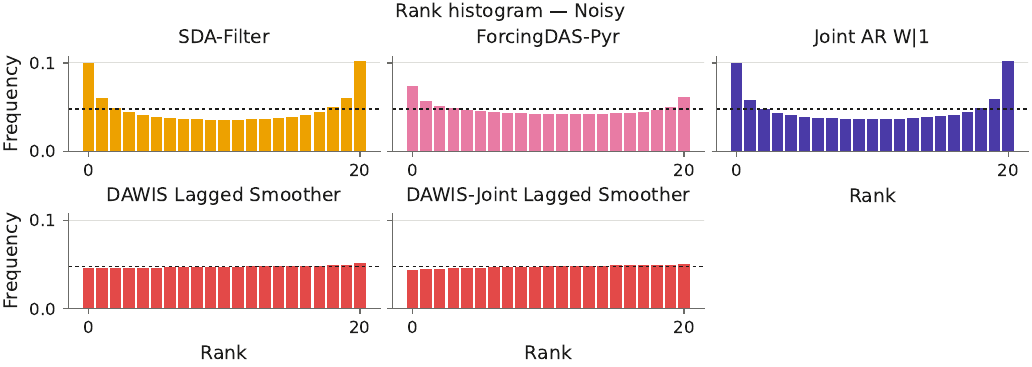}
  \caption{%
    Rank histograms of the smoothing methods for \texttt{Noisy}.
  }
  \label{fig:rank-smoothing-noisy}
\end{figure}

\begin{figure}[H]
  \centering
  \includegraphics[width=\linewidth]{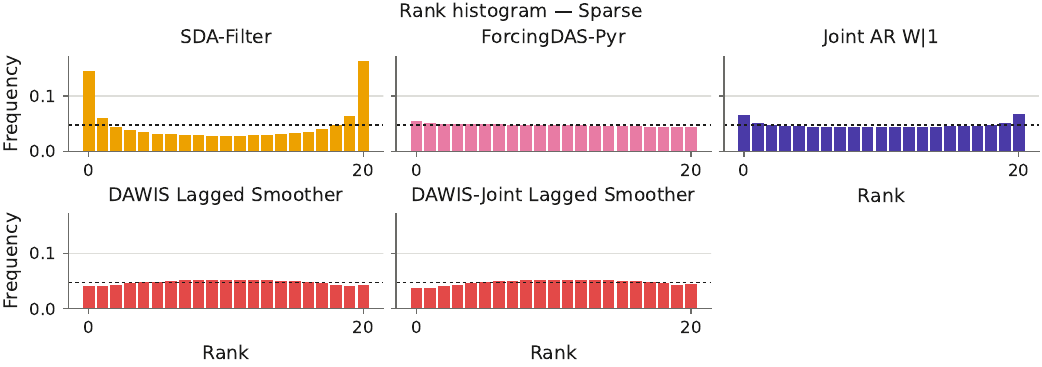}
  \caption{%
    Rank histograms of the smoothing methods for \texttt{Sparse}.
  }
  \label{fig:rank-smoothing-sparse}
\end{figure}

\begin{figure}[H]
  \centering
  \includegraphics[width=\linewidth]{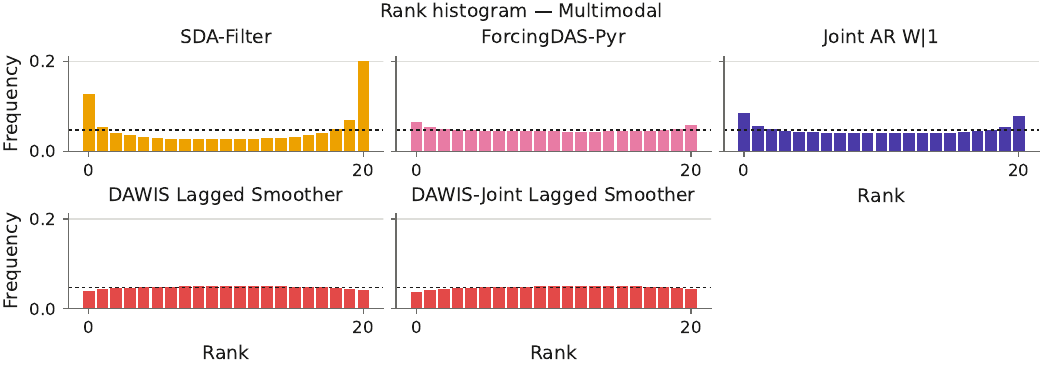}
  \caption{%
    Rank histograms of the smoothing methods for \texttt{Multimodal}.
  }
  \label{fig:rank-smoothing-multimodal}
\end{figure}

\begin{figure}[H]
  \centering
  \includegraphics[width=\linewidth]{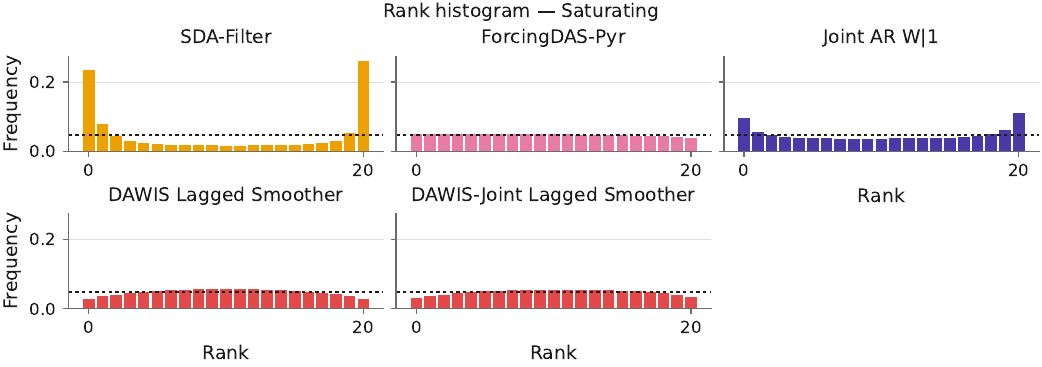}
  \caption{%
    Rank histograms of the smoothing methods for \texttt{Saturating}.
  }
  \label{fig:rank-smoothing-saturating}
\end{figure}

\subsection{Power spectra}
Pointwise scores such as CRPS cannot tell whether a method blurs the fields or adds spurious small-scale noise. The radially averaged power spectra can (\cref{fig:spectra-filtering,fig:spectra-smoothing}), and the log spectral distance (LSD, \cref{fig:lsd-time-filtering}) is used to aggregate and tracks the discrepancy over time. DAWIS matches the true spectrum across the resolved range in every setting. EnSF and FlowDAS have too much power at high wavenumbers on SQG. LETKF has too much power at intermediate and high wavenumbers on \texttt{Multimodal}, consistent with the grid-scale noise in \cref{fig:evolution-multimodal}. SDA-Filter has too little power at all scales on \texttt{Saturating}. On \texttt{SEVIR}, EnSF underestimates the large scales and overestimates the small ones. The LSD shows that these errors are not transient. The spectral error grows over time for LETKF on \texttt{Multimodal}, FlowDAS and Joint AR 1$|$W on \texttt{Sparse}, and SDA-Filter on \texttt{Saturating}, while DAWIS stays lowest and constant. After smoothing, every method reproduces the spectrum except SDA-Filter, which lacks power on \texttt{Saturating} and has too much at the smallest \texttt{SEVIR} scales.

\begin{figure}[H]
  \centering
  \includegraphics[width=\linewidth]{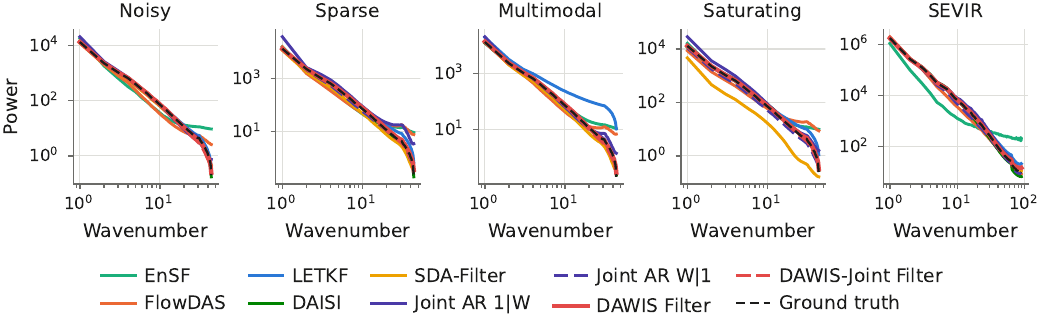}
  \caption{%
    Radially averaged power spectra of the filtering ensemble.
  }
  \label{fig:spectra-filtering}
\end{figure}

\begin{figure}[H]
  \centering
  \includegraphics[width=\linewidth]{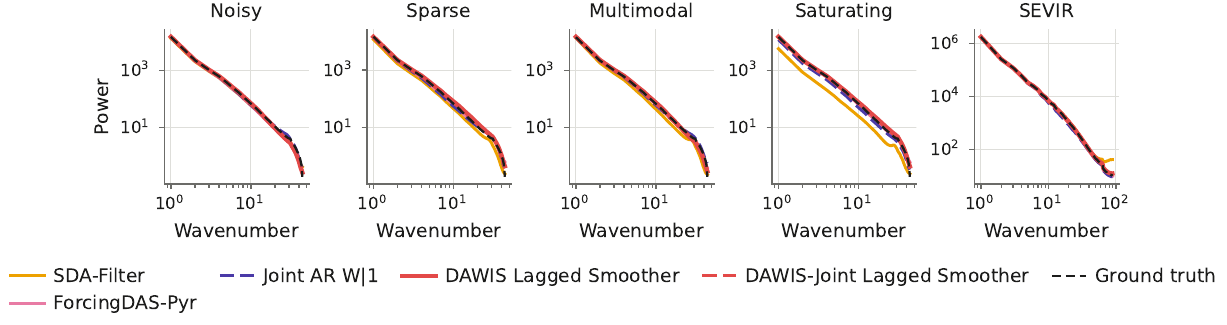}
  \caption{%
    Radially averaged power spectra of the lagged smoothing ensemble.
  }
  \label{fig:spectra-smoothing}
\end{figure}

\begin{figure}[H]
  \centering
  \includegraphics[width=\linewidth]{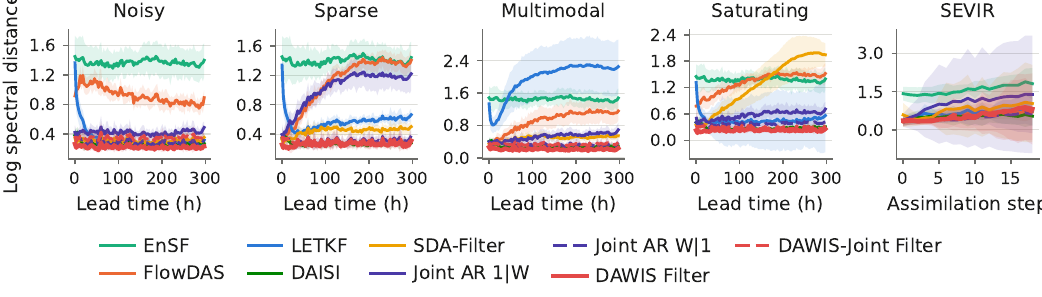}
  \caption{%
    Log spectral distance over the assimilation window for the filtering methods.
  }
  \label{fig:lsd-time-filtering}
\end{figure}

\section{Ablation Experiments}
Here we show ablation experiments. In the main SQG experiments, the forecast model is the model that generated the data, so \cref{sec:model_error} tests robustness to a misspecified forward model. \Cref{sec:dawis_soft_block} relaxes the fixed boundary of the block smoother. In both cases DAWIS degrades gracefully. Under model error, the DAWIS Filter CRPS rises by at most $0.15$. DAWIS-Joint does not integrate a forecast model and stays at or below its well-specified CRPS, overtaking DAWIS on \texttt{Sparse} and \texttt{Saturating}.

\subsection{Comparison with En4DVar under Model error}\label{sec:model_error}
Variational methods such as 4DVar optimize the initial state of an assimilation window so that the forecast model's trajectory fits all observations in the window. With an exact adjoint this is highly effective, but it yields a single MAP trajectory. The Ensemble of Data Assimilations \citep[EDA;][]{bonavita2012en4dvar} runs several 4DVar solves from perturbed backgrounds and observations to obtain an ensemble; we refer to this as \emph{En4DVar}.

\paragraph{Why model error.}
In the SQG experiments above, the numerical model available at test time is the model that generated the data. Methods that integrate the dynamics therefore have access to the exact governing equations, and 4DVar in particular can differentiate through them. In this perfect-model setting a strong-constraint 4DVar with an exact adjoint is close to optimal, which says little about the regime DA is used in, where the forecast model is imperfect. We therefore repeat the experiments with a misspecified model. The ground-truth trajectories are generated by an SQG simulation at $256 \times 256$ resolution and sprectrally truncated onto the $64 \times 64$ grid used for assimilation. Both the numerical forward model used by En4DVar and the trajectories used to train the DAWIS prior come from the $64 \times 64$ simulator, so the unresolved scales of the high-resolution run constitute model error that neither method has seen. Observation operators, noise levels and trajectories are otherwise identical to the corresponding experiments above; the resulting settings are marked (ds) in \cref{tab:ds}.

\paragraph{En4DVar implementation.}
Strong-constraint 4DVar minimizes, over the initial state $\vx_0$ of a window of $W$ steps,
\begin{align}
  J(\vx_0) = \tfrac12 (\vx_0 - \vx^{\rm b})^\top B^{-1} (\vx_0 - \vx^{\rm b})
           + \tfrac12 \sum_{t=0}^{W-1} \big\|\vy_t - H(M_t \vx_0)\big\|^2_{R^{-1}},
\end{align}
where $\vx^{\rm b}$ is the background, $M_t$ the forecast model integrated $t$ steps and $R = \sigma_y^2 \mat I$ \cite{bannister2008review}. Following \citet{solvik20254d}, the adjoint is obtained by backpropagation through the differentiable SQG solver, which is exact up to floating point; operational systems instead rely on approximate tangent-linear and adjoint models. The cost is minimized with L-BFGS. The background covariance $B$ is a static, homogeneous and isotropic spectral diagonal estimated from climatological anomalies \citep{bannister2008review}. We additionally tried a localized ensemble covariance of the kind used by LETKF, but got similar or worse results. Windows are cycled without overlap, the background of each window being the previous analysis propagated one step, and the analysis trajectory is the deterministic rollout from the optimized $\vx_0$. Consequently the first state of a window is estimated with $W - 1$ observations of look-ahead and the last with none, whereas the DAWIS Lagged Smoother has a constant lag of $w$. The ensemble has $J = 20$ members whose backgrounds are initialized as for LETKF. Observations are shared across members, so the ensemble spread stems from the initial perturbations alone. We use $W = w+1$ so that En4DVar and DAWIS see the same window of observations per cycle.

\paragraph{Results.}
The results in \cref{tab:ds} show that DAWIS consistently outperforms En4DVar in both CRPS and RMSE across all four observation settings. The DAWIS-Joint variants achieve comparable performance to the standard DAWIS variants while not requiring a separate forecasting model. The difference is particularly pronounced for the smoothing methods, where both DAWIS variants substantially reduce CRPS and RMSE relative to En4DVar. The SSR results further illustrate the effect of model misspecification on calibration. En4DVar is substantially underdispersive in the \texttt{Noisy} setting and overdispersive in \texttt{Sparse} and \texttt{Multimodal}, while the DAWIS variants remain closer to the target value of one across the different observation scenarios. Overall, these results indicate that DAWIS maintains accurate and comparatively well-calibrated uncertainty estimates even when the dynamics used during assimilation do not match those generating the data.

\begin{table}[h]
\centering
\small
\caption{The CRPS, RMSE and SSR under forward-model error.
The truth is a higher-resolution SQG run downsampled onto the $64\times64$ grid, so the model each method integrates is not the model that generated the data; the observation operator, the noise and the trajectories are those of the corresponding experiment above.
We display the mean and standard deviation across 10 independent trajectories, averaged over the last 20 steps for filtering methods and over the whole trajectory for smoothing methods.
The best score for each experiment is highlighted in \textbf{bold} and the second best with an \underline{underline}.
For SSR, best is closest to 1.}
\label{tab:ds}
\setlength{\tabcolsep}{4pt}
\begin{tabular}{l c c c c}
\toprule
\textbf{Method}
& \textbf{Noisy (ds)}
& \textbf{Sparse (ds)}
& \textbf{Multimodal (ds)}
& \textbf{Saturating (ds)} \\

\midrule
\multicolumn{5}{c}{\textbf{CRPS} $\downarrow$} \\
\midrule

\multicolumn{5}{l}{\textit{Filtering}} \\

DAWIS Filter
& $\mathbf{1.12}{\scriptstyle \pm0.14}$
& $\underline{1.22}{\scriptstyle \pm0.16}$
& $\mathbf{1.13}{\scriptstyle \pm0.13}$
& $\underline{1.01}{\scriptstyle \pm0.15}$ \\

DAWIS-Joint Filter
& $\underline{1.13}{\scriptstyle \pm0.13}$
& $\mathbf{1.19}{\scriptstyle \pm0.17}$
& $\underline{1.18}{\scriptstyle \pm0.13}$
& $\mathbf{0.95}{\scriptstyle \pm0.16}$ \\

\addlinespace[0.3em]
\multicolumn{5}{l}{\textit{Smoothing}} \\

En4DVar
& $1.35{\scriptstyle \pm0.04}$
& $1.33{\scriptstyle \pm0.17}$
& $1.16{\scriptstyle \pm0.04}$
& $0.94{\scriptstyle \pm0.08}$ \\

DAWIS Lagged Smoother
& $\mathbf{0.98}{\scriptstyle \pm0.05}$
& $\underline{1.06}{\scriptstyle \pm0.06}$
& $\mathbf{1.02}{\scriptstyle \pm0.05}$
& $\underline{0.83}{\scriptstyle \pm0.06}$ \\

DAWIS-Joint Lagged Smoother
& $\underline{1.04}{\scriptstyle \pm0.05}$
& $\mathbf{1.05}{\scriptstyle \pm0.06}$
& $\underline{1.10}{\scriptstyle \pm0.06}$
& $\mathbf{0.81}{\scriptstyle \pm0.07}$ \\

\midrule
\multicolumn{5}{c}{\textbf{RMSE} $\downarrow$} \\
\midrule

\multicolumn{5}{l}{\textit{Filtering}} \\

DAWIS Filter
& $\mathbf{2.08}{\scriptstyle \pm0.26}$
& $\underline{2.40}{\scriptstyle \pm0.33}$
& $\mathbf{2.15}{\scriptstyle \pm0.24}$
& $\underline{1.95}{\scriptstyle \pm0.28}$ \\

DAWIS-Joint Filter
& $\underline{2.13}{\scriptstyle \pm0.24}$
& $\mathbf{2.36}{\scriptstyle \pm0.35}$
& $\underline{2.27}{\scriptstyle \pm0.25}$
& $\mathbf{1.87}{\scriptstyle \pm0.31}$ \\

\addlinespace[0.3em]
\multicolumn{5}{l}{\textit{Smoothing}} \\

En4DVar
& $2.35{\scriptstyle \pm0.08}$
& $2.65{\scriptstyle \pm0.35}$
& $2.17{\scriptstyle \pm0.09}$
& $1.80{\scriptstyle \pm0.15}$ \\

DAWIS Lagged Smoother
& $\mathbf{1.84}{\scriptstyle \pm0.09}$
& $\underline{2.06}{\scriptstyle \pm0.12}$
& $\mathbf{1.94}{\scriptstyle \pm0.09}$
& $\underline{1.63}{\scriptstyle \pm0.10}$ \\

DAWIS-Joint Lagged Smoother
& $\underline{1.96}{\scriptstyle \pm0.10}$
& $\mathbf{2.04}{\scriptstyle \pm0.12}$
& $\underline{2.10}{\scriptstyle \pm0.12}$
& $\mathbf{1.59}{\scriptstyle \pm0.12}$ \\

\midrule
\multicolumn{5}{c}{\textbf{SSR} ($\to 1$)} \\
\midrule

\multicolumn{5}{l}{\textit{Filtering}} \\

DAWIS Filter
& $\underline{0.79}{\scriptstyle \pm0.01}$
& $\mathbf{1.06}{\scriptstyle \pm0.02}$
& $\underline{0.87}{\scriptstyle \pm0.03}$
& $\underline{0.92}{\scriptstyle \pm0.01}$ \\

DAWIS-Joint Filter
& $\mathbf{0.93}{\scriptstyle \pm0.02}$
& $\underline{1.09}{\scriptstyle \pm0.03}$
& $\mathbf{0.98}{\scriptstyle \pm0.03}$
& $\mathbf{1.07}{\scriptstyle \pm0.03}$ \\

\addlinespace[0.3em]
\multicolumn{5}{l}{\textit{Smoothing}} \\

En4DVar
& $0.77{\scriptstyle \pm0.01}$
& $1.36{\scriptstyle \pm0.03}$
& $1.15{\scriptstyle \pm0.07}$
& $\mathbf{1.00}{\scriptstyle \pm0.06}$ \\

DAWIS Lagged Smoother
& $\underline{0.92}{\scriptstyle \pm0.01}$
& $\mathbf{1.12}{\scriptstyle \pm0.01}$
& $\mathbf{0.98}{\scriptstyle \pm0.02}$
& $\underline{1.04}{\scriptstyle \pm0.01}$ \\

DAWIS-Joint Lagged Smoother
& $\mathbf{1.02}{\scriptstyle \pm0.02}$
& $\underline{1.17}{\scriptstyle \pm0.02}$
& $\underline{1.06}{\scriptstyle \pm0.03}$
& $1.14{\scriptstyle \pm0.01}$ \\

\bottomrule
\end{tabular}
\end{table}

\subsection{DAWIS Soft Block Smoother (adjust endpoints)}\label{sec:dawis_soft_block}
The DAWIS Block Smoother holds the two boundary states fixed, which is what makes each update an exact conditional sampling (\cref{prop:invariance}). A natural relaxation also lets the boundary move, inverting it to a turning point $\taumin^{\partial\mathcal B}<1$ (we use $0.9$) while the interior is inverted to $\taumin$ as before. The update is then no longer a Gibbs step on $\pi_{\mathcal B}$, so we make no invariance claim and instead evaluate empirical performance. \Cref{tab:block} shows that on SQG the soft variant performs within reported precision of the standard one on \texttt{Noisy}, \texttt{Sparse} and \texttt{Multimodal}, and slightly better on \texttt{Saturating} (CRPS $0.65$ vs.\ $0.68$). Holding the boundary therefore costs little in practice, while potentially improving empirical performance in certain settings.
\begin{table}[H]
\centering
\small
\caption{The CRPS, RMSE and SSR for experiments on SQG and \texttt{SEVIR}.
We display the mean and standard deviation across 10 independent trajectories, averaged over the whole trajectory.
The best score for each experiment is highlighted in \textbf{bold} and the second best with an \underline{underline}.
For SSR, best is closest to 1.}
\label{tab:block}
\setlength{\tabcolsep}{4pt}
\begin{tabular}{l c c c c c}
\toprule
\textbf{Method}
& \textbf{Noisy}
& \textbf{Sparse}
& \textbf{Multimodal}
& \textbf{Saturating}
& \textbf{SEVIR} \\

\midrule
\multicolumn{6}{c}{\textbf{CRPS} $\downarrow$} \\
\midrule

Block Gibbs Smoother
& $0.92{\scriptstyle \pm0.06}$
& $0.99{\scriptstyle \pm0.07}$
& $0.90{\scriptstyle \pm0.05}$
& $0.86{\scriptstyle \pm0.09}$
& $3.58{\scriptstyle \pm2.35}$ \\

DAWIS Block Smoother
& $0.90{\scriptstyle \pm0.05}$
& $0.98{\scriptstyle \pm0.06}$
& $0.89{\scriptstyle \pm0.05}$
& $0.68{\scriptstyle \pm0.06}$
& $3.18{\scriptstyle \pm2.17}$ \\

DAWIS Soft Block Smoother
& $0.91{\scriptstyle \pm0.05}$
& $0.97{\scriptstyle \pm0.06}$
& $0.89{\scriptstyle \pm0.05}$
& $0.65{\scriptstyle \pm0.06}$
& $3.15{\scriptstyle \pm2.19}$ \\

\midrule
\multicolumn{6}{c}{\textbf{RMSE} $\downarrow$} \\
\midrule

Block Gibbs Smoother
& $1.70{\scriptstyle \pm0.10}$
& $1.92{\scriptstyle \pm0.12}$
& $1.68{\scriptstyle \pm0.09}$
& $1.66{\scriptstyle \pm0.15}$
& $9.27{\scriptstyle \pm5.63}$ \\

DAWIS Block Smoother
& $1.67{\scriptstyle \pm0.09}$
& $1.92{\scriptstyle \pm0.12}$
& $1.68{\scriptstyle \pm0.08}$
& $1.36{\scriptstyle \pm0.11}$
& $8.93{\scriptstyle \pm5.45}$ \\

DAWIS Soft Block Smoother
& $1.68{\scriptstyle \pm0.09}$
& $1.90{\scriptstyle \pm0.12}$
& $1.67{\scriptstyle \pm0.08}$
& $1.32{\scriptstyle \pm0.11}$
& $8.92{\scriptstyle \pm5.47}$ \\

\midrule
\multicolumn{6}{c}{\textbf{SSR} ($\to 1$)} \\
\midrule

Block Gibbs Smoother
& $0.89{\scriptstyle \pm0.01}$
& $1.03{\scriptstyle \pm0.04}$
& $0.94{\scriptstyle \pm0.01}$
& $\mathbf{1.00}{\scriptstyle \pm0.02}$
& $1.16{\scriptstyle \pm0.59}$ \\

DAWIS Block Smoother
& $0.92{\scriptstyle \pm0.01}$
& $1.06{\scriptstyle \pm0.04}$
& $0.95{\scriptstyle \pm0.01}$
& $1.12{\scriptstyle \pm0.02}$
& $1.31{\scriptstyle \pm0.67}$ \\

DAWIS Soft Block Smoother
& $0.90{\scriptstyle \pm0.01}$
& $1.05{\scriptstyle \pm0.03}$
& $0.95{\scriptstyle \pm0.01}$
& $1.11{\scriptstyle \pm0.02}$
& $1.33{\scriptstyle \pm0.72}$ \\

\bottomrule
\end{tabular}
\end{table}

\section{State fields}
We show qualitative examples for each setting. For every method, the panels show the ensemble mean, one member, the ensemble spread and the absolute error of the mean: for filtering at the final assimilation step (\cref{fig:fields-filtering-noisy,fig:fields-filtering-sparse,fig:fields-filtering-multimodal,fig:fields-filtering-saturating,fig:fields-filtering-sevir}), and for smoothing halfway through the trajectory (\cref{fig:fields-smoothing-noisy,fig:fields-smoothing-sparse,fig:fields-smoothing-multimodal,fig:fields-smoothing-saturating,fig:fields-smoothing-sevir}). The failure modes seen in the aggregate scores show up clearly here. On \texttt{Multimodal}, LETKF's analysis turns into grid-scale noise. The Joint AR 1$|$W members are individually plausible but disagree with each other, so their mean collapses into a featureless band. FlowDAS members have noisy, unphysical patches. \Cref{fig:evolution-multimodal} follows the \texttt{Multimodal} filter mean over time. LETKF loses the flow structure within $48$\,h. DAISI stays coherent but gradually loses the high-frequency features, drifting towards an overly smooth field. The DAWIS filter tracks the high-frequency features better throughout the trajectory.

\begin{figure*}[tb]
  \centering
  \includegraphics[height=\textheight]{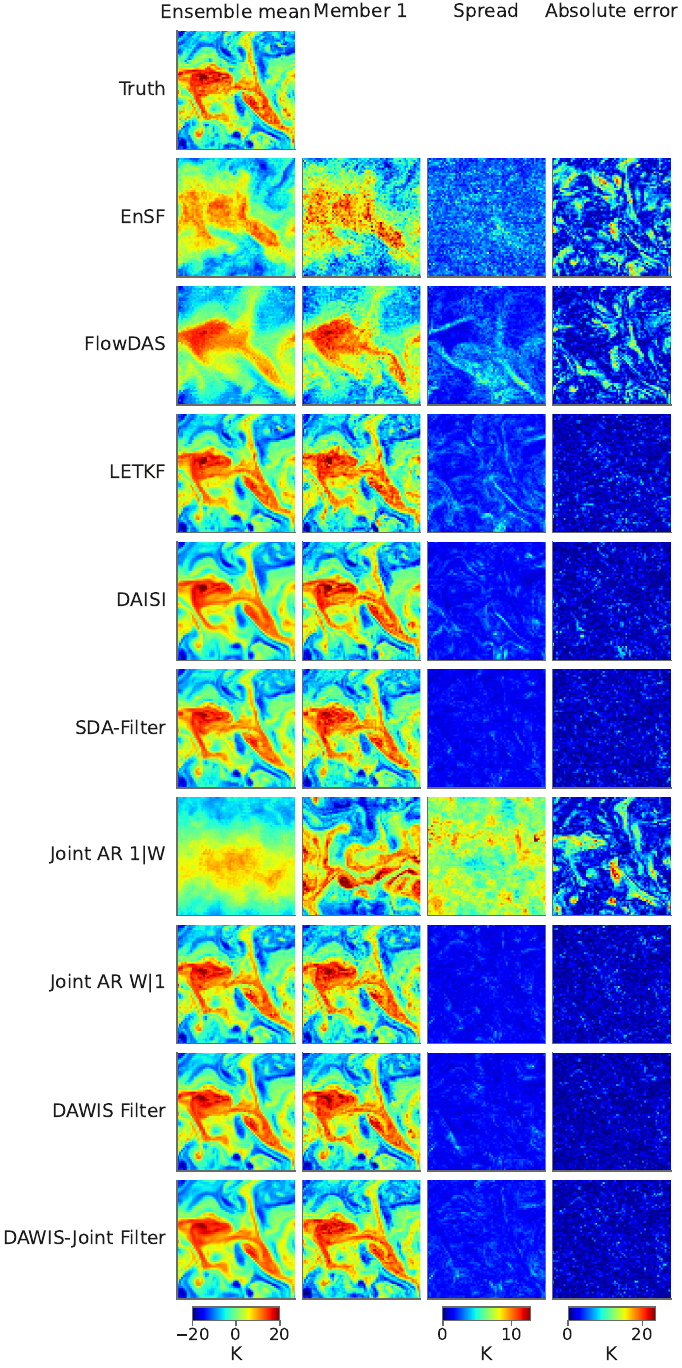}
  \caption{%
    Filtering states at the end of the trajectory for SQG \texttt{Noisy}.
  }
  \label{fig:fields-filtering-noisy}
\end{figure*}

\begin{figure*}[tb]
  \centering
  \includegraphics[height=\textheight]{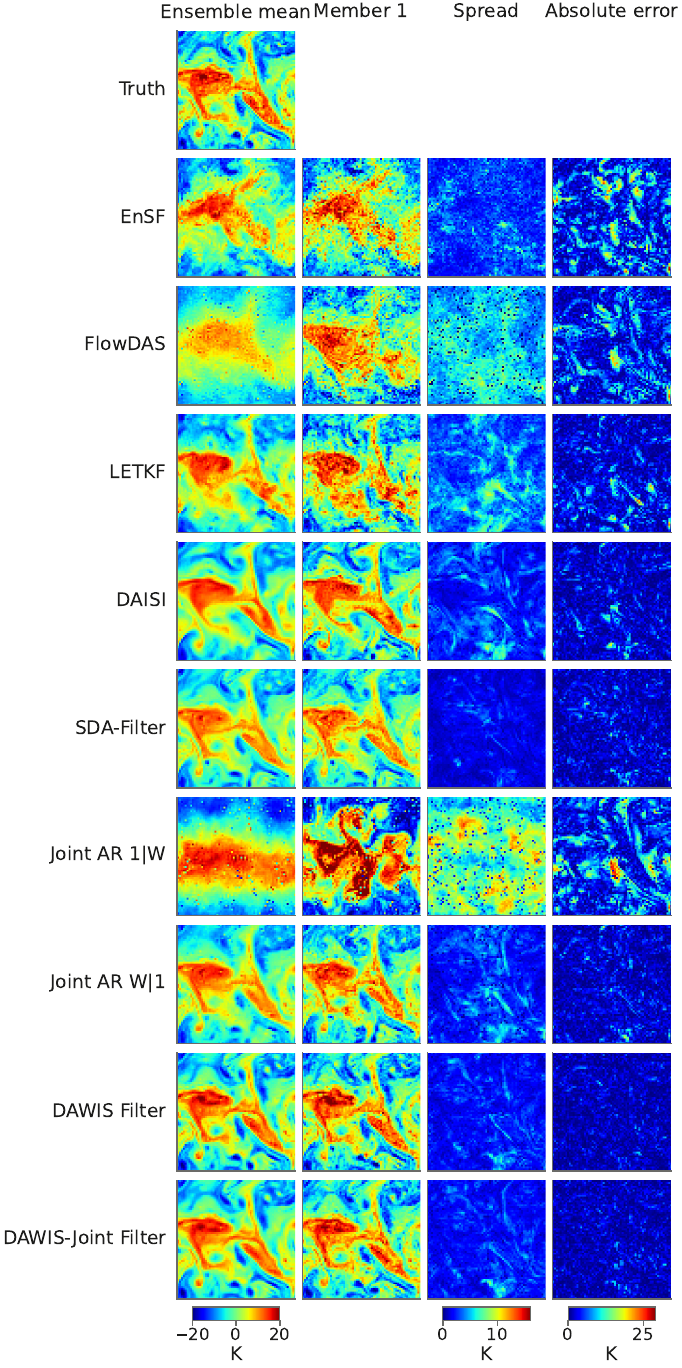}
  \caption{%
    Filtering states at the end of the trajectory for SQG \texttt{Sparse}.
  }
  \label{fig:fields-filtering-sparse}
\end{figure*}

\begin{figure*}[tb]
  \centering
  \includegraphics[height=\textheight]{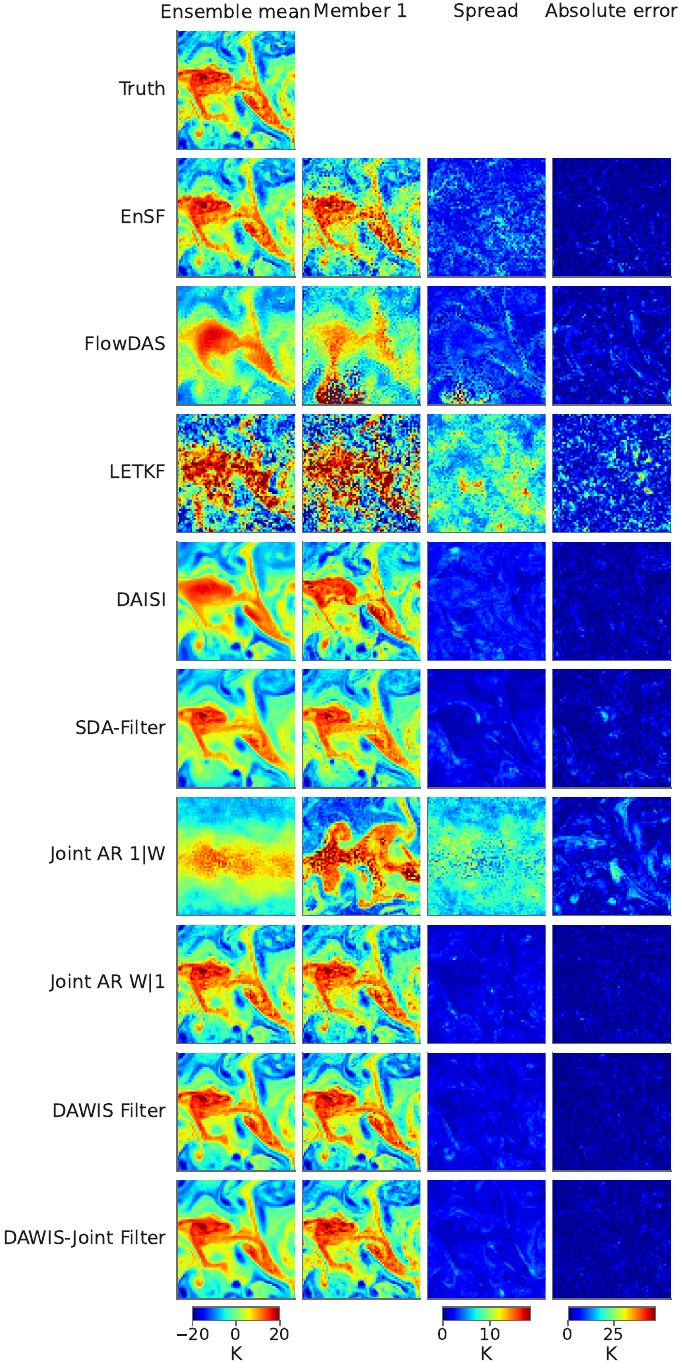}
  \caption{%
    Filtering states at the end of the trajectory for SQG \texttt{Multimodal}.
  }
  \label{fig:fields-filtering-multimodal}
\end{figure*}

\begin{figure*}[tb]
  \centering
  \includegraphics[height=\textheight]{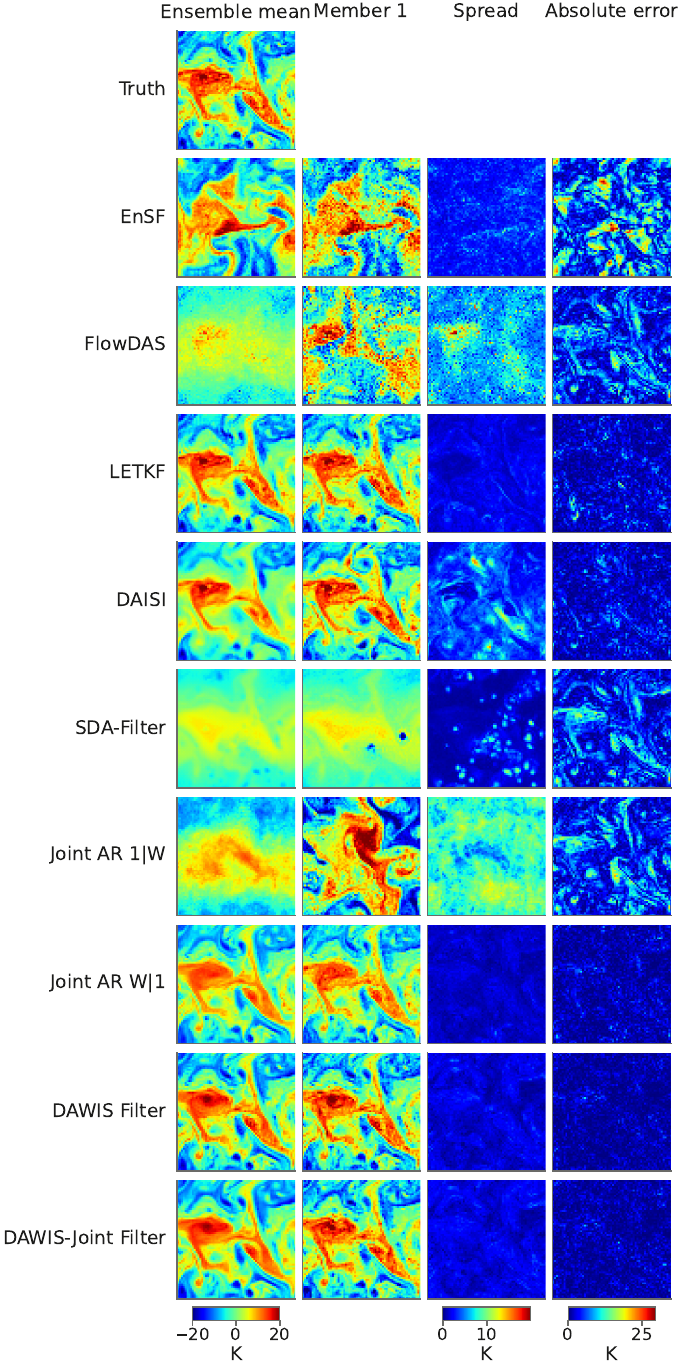}
  \caption{%
    Filtering states at the end of the trajectory for SQG \texttt{Saturating}.
  }
  \label{fig:fields-filtering-saturating}
\end{figure*}

\begin{figure*}[tb]
  \centering
  \includegraphics[height=\textheight]{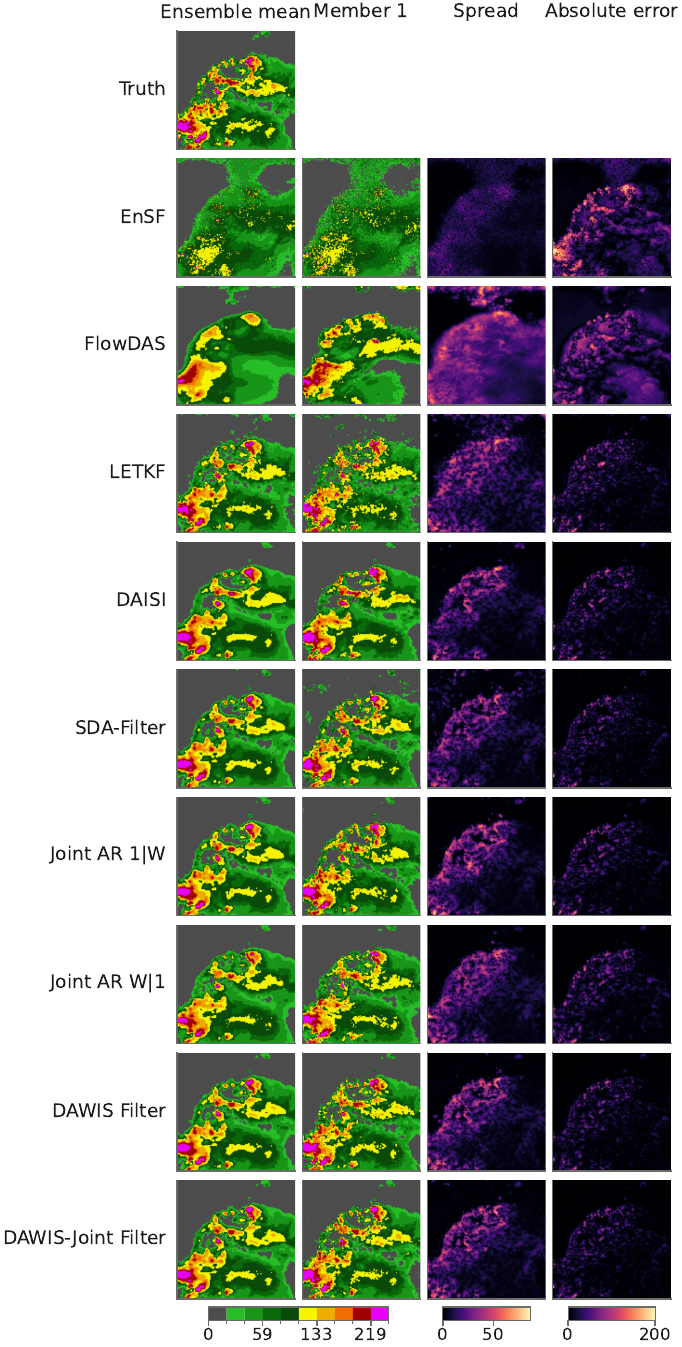}
  \caption{%
    Filtering states at the end of the trajectory for \texttt{SEVIR}.
  }
  \label{fig:fields-filtering-sevir}
\end{figure*}

\begin{figure*}[tb]
  \centering
  \includegraphics[width=0.95\linewidth]{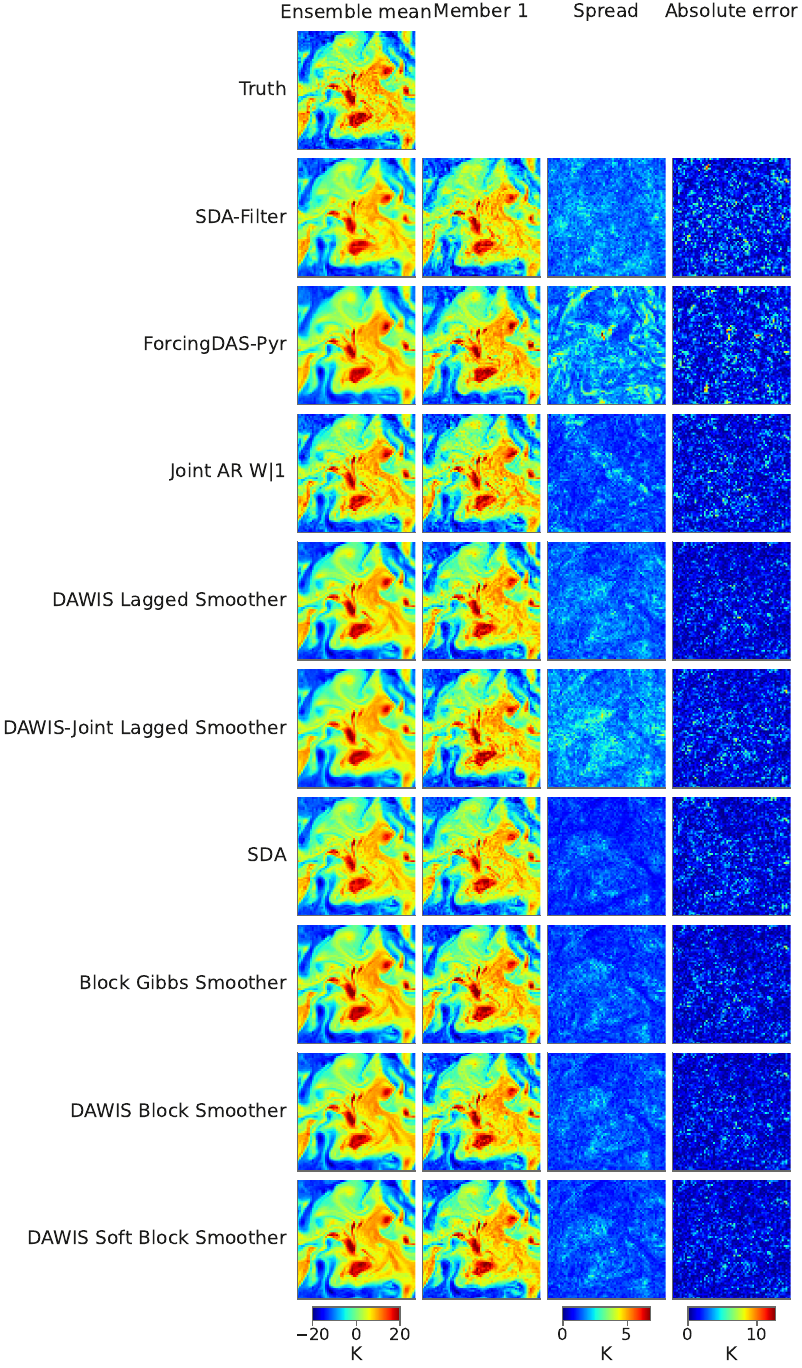}
  \caption{%
    Smoothing states halfway through the trajectory for SQG  \texttt{Noisy}.
  }
  \label{fig:fields-smoothing-noisy}
\end{figure*}

\begin{figure*}[tb]
  \centering
  \includegraphics[width=0.95\linewidth]{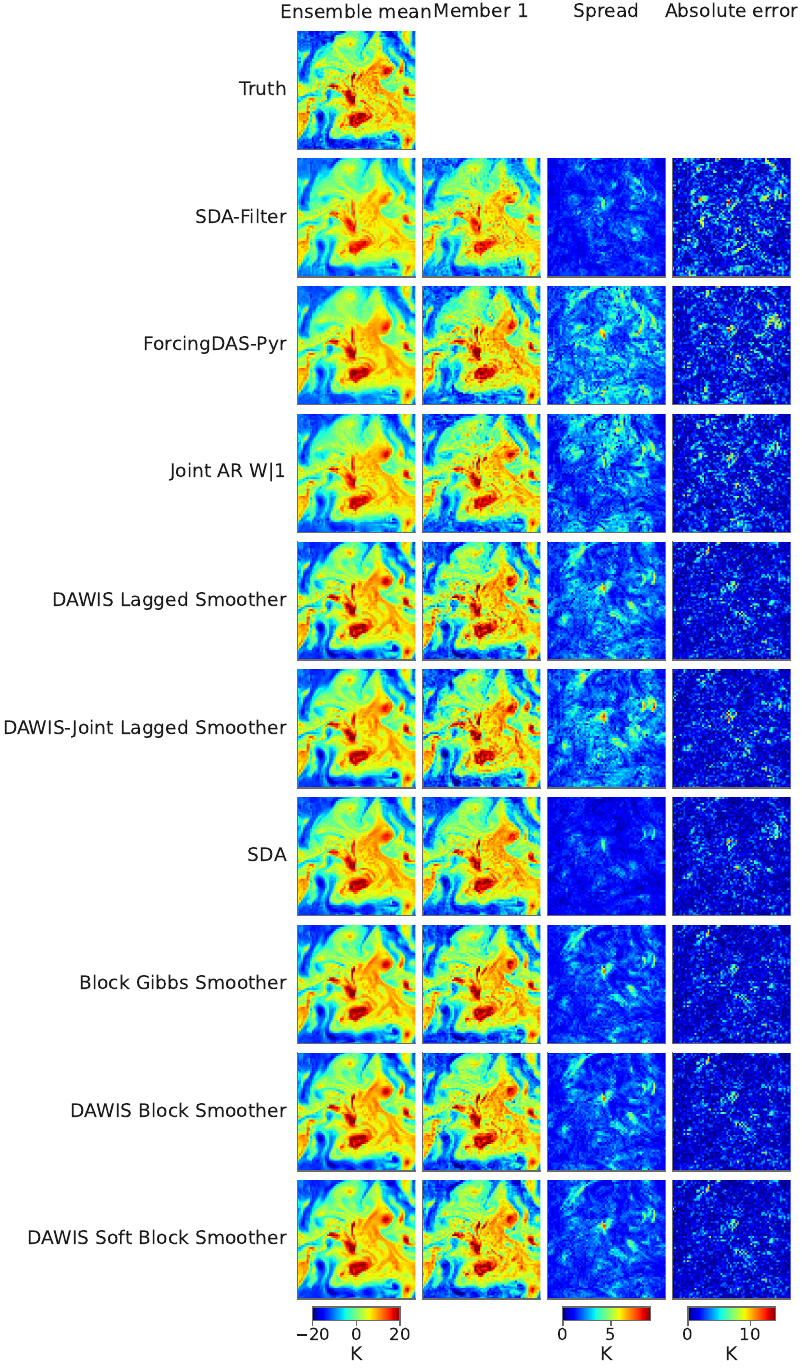}
  \caption{%
    Smoothing states halfway through the trajectory for SQG  \texttt{Sparse}.
  }
  \label{fig:fields-smoothing-sparse}
\end{figure*}

\begin{figure*}[tb]
  \centering
  \includegraphics[width=0.95\linewidth]{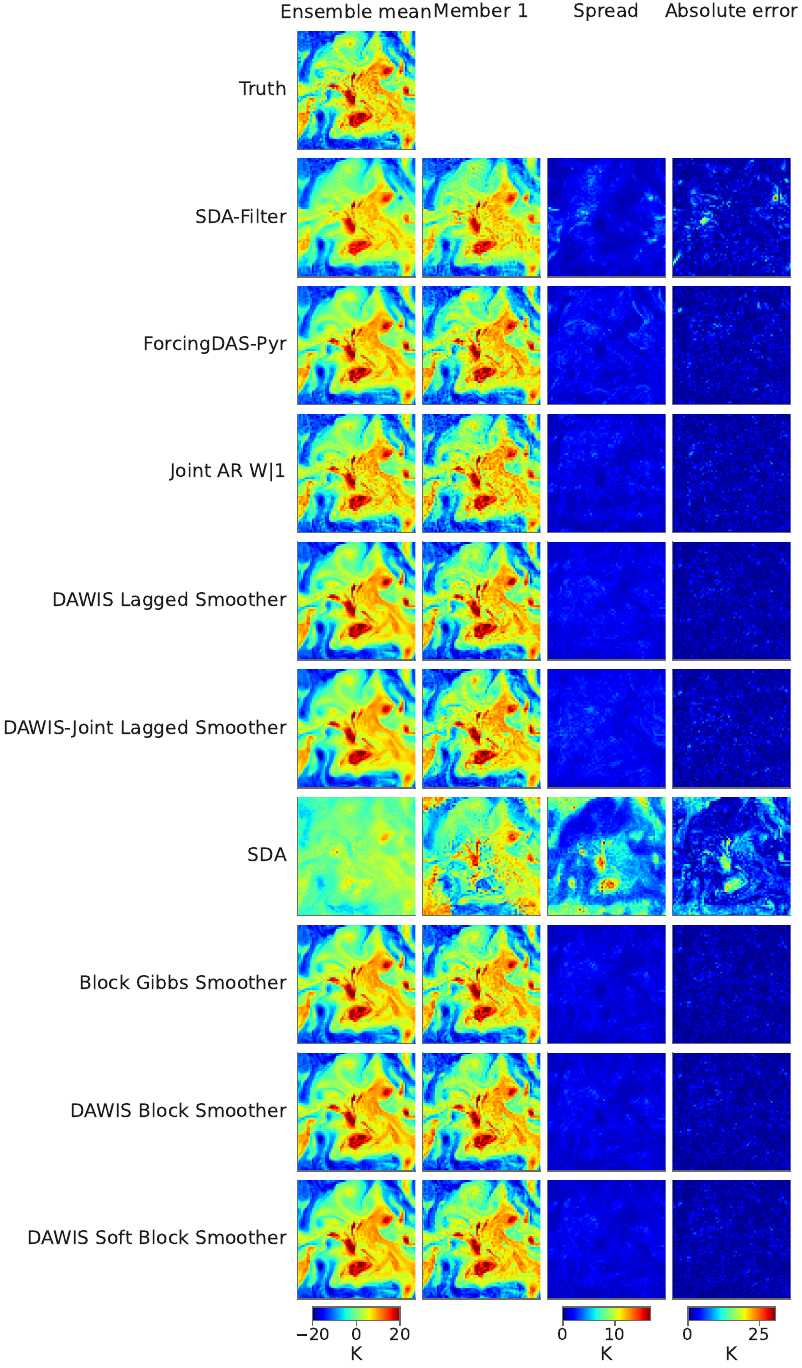}
  \caption{%
    Smoothing states halfway through the trajectory for SQG  \texttt{Multimodal}.
  }
  \label{fig:fields-smoothing-multimodal}
\end{figure*}

\begin{figure*}[tb]
  \centering
  \includegraphics[width=0.95\linewidth]{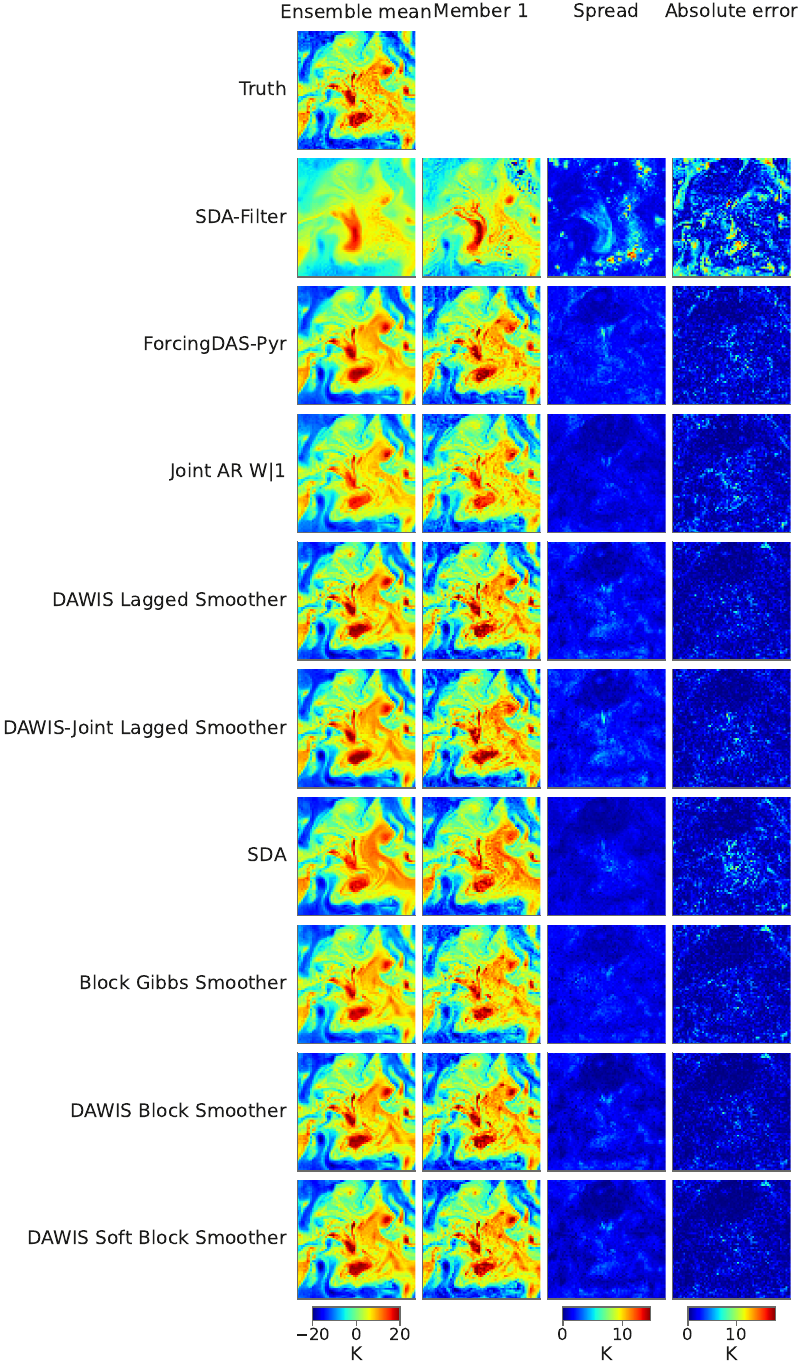}
  \caption{%
    Smoothing states halfway through the trajectory for SQG  \texttt{Saturating}.
  }
  \label{fig:fields-smoothing-saturating}
\end{figure*}

\begin{figure*}[tb]
  \centering
  \includegraphics[width=0.95\linewidth]{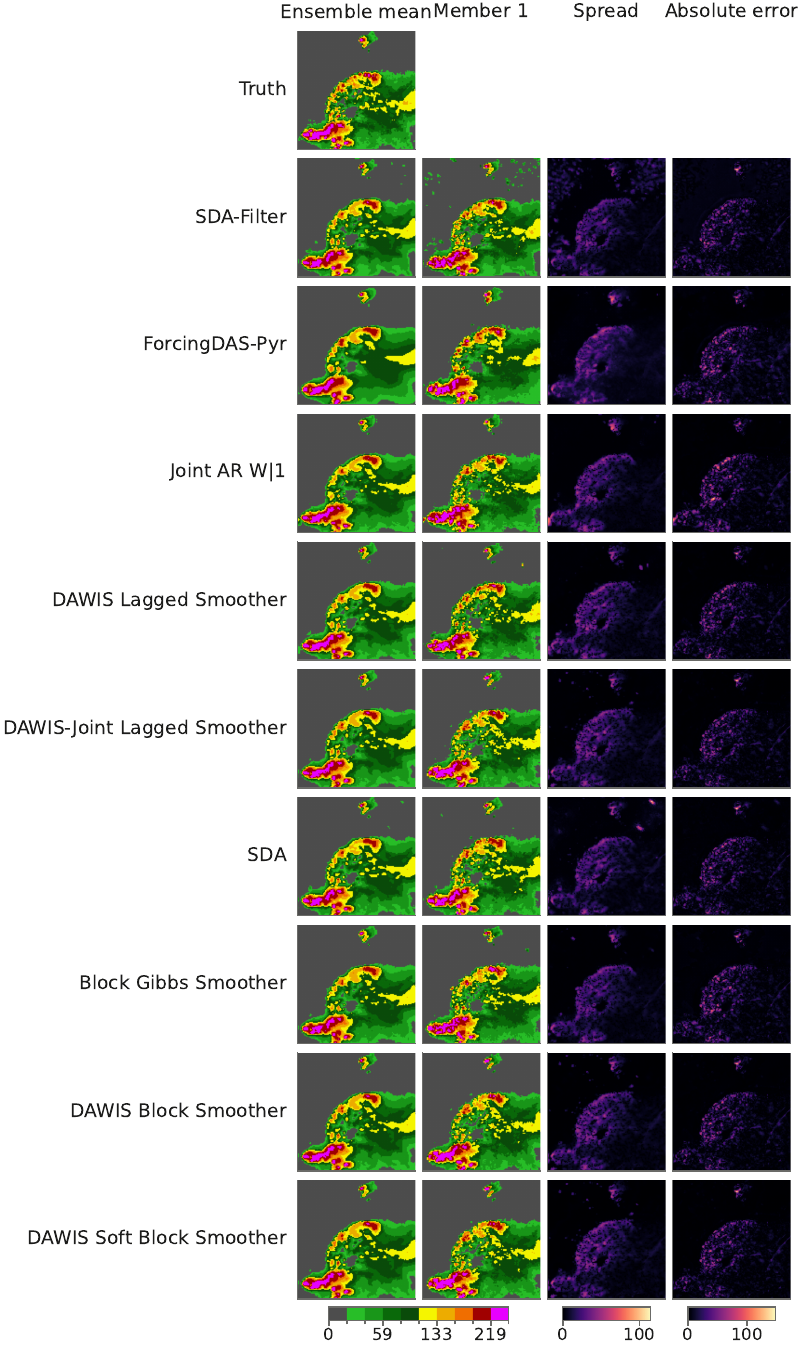}
  \caption{%
    Smoothing states halfway through the trajectory for \texttt{SEVIR}.
  }
  \label{fig:fields-smoothing-sevir}
\end{figure*}

\begin{figure*}[tb]
  \centering
  \includegraphics[width=\linewidth]{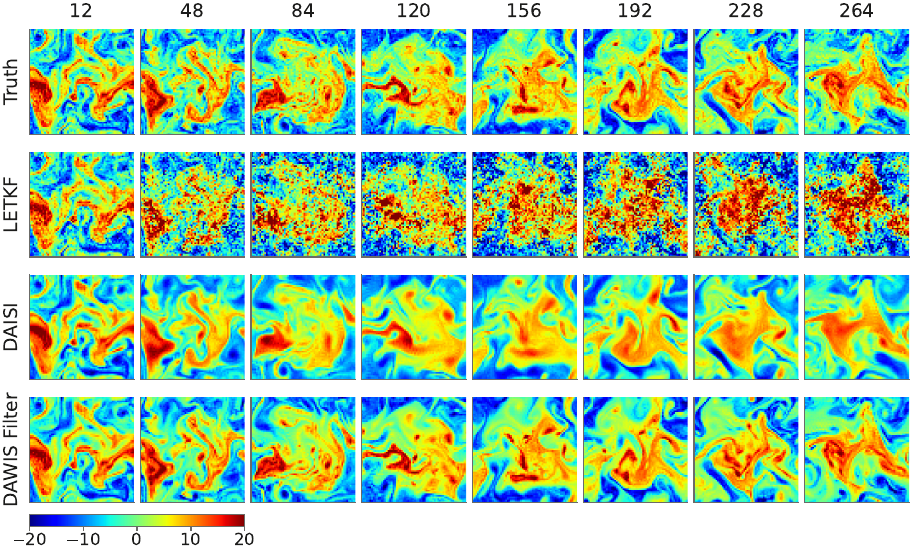}
  \caption{%
    The evolution of the ensemble mean of the filtering distribution over the assimilation window for SQG \texttt{Multimodal}.
  }
  \label{fig:evolution-multimodal}
\end{figure*}

\end{document}